\documentclass[12pt]{article}
\usepackage{amsmath}
\usepackage{amsfonts,amsmath,amssymb,amsthm}
\usepackage{graphicx,psfrag,epsf}
\usepackage{enumerate}
 \usepackage[numbers]{natbib}
\RequirePackage[colorlinks,citecolor=blue,linkcolor=blue,urlcolor=blue]{hyperref}

\usepackage{url} 

\newcommand{\blind}{0}

\usepackage{bm}
\usepackage[capitalise]{cleveref}

\usepackage[shortlabels]{enumitem}

\newcommand{\dataset}{{\cal D}}

\def\R{\mathbb{R}}
\def\E{\mathbb{E}}
\def\one{\mathds{1}}
\def\P{\mathrm{Pr}}
\def\Cov{\mathrm{cov}}
\def\V{\mathrm{var}}
\def\bz{\mathbf{Z}}

\usepackage{amsmath,amssymb}
\usepackage[capitalise]{cleveref}
\usepackage{dsfont}

\newtheorem{definition}{Definition}
\newtheorem{lemma}{Lemma}
\newtheorem{example}{Example}

\newtheorem{remark}{Remark}

\newtheorem{proposition}{Proposition}
\newtheorem{theorem}{Theorem}

\begin{document}

	\def\spacingset#1{\renewcommand{\baselinestretch}%
		{#1}\small\normalsize} \spacingset{1}

	
	\if0\blind
	{
		\title{\bf Asymptotic Distributions and Rates of Convergence for Random Forests via Generalized U-statistics}
		\author{Wei Peng, Tim Coleman, \& Lucas Mentch\thanks{
				This work was partially supported by \textit{NSF DMS-1712041}}  \hspace{.2cm}\\
			Department of Statistics, University of Pittsburgh}
		\maketitle
	} \fi
	
	\if1\blind
	{
		\bigskip
		\bigskip
		\bigskip
		\begin{center}
			{\LARGE\bf Asymptotic Distributions and Rates of Convergence for Random Forests via Generalized U-statistics}
		\end{center}
		\medskip
	} \fi
	
	\bigskip
	\begin{abstract}
		Random forests are among the most popular off-the-shelf supervised learning algorithms.  Despite their well-documented empirical success, however, until recently, few theoretical results were available to describe their performance and behavior.  In this work we push beyond recent work on consistency and asymptotic normality by establishing rates of convergence for random forests and other supervised learning ensembles.  We develop the notion of generalized U-statistics and show that within this framework, random forest predictions can remain asymptotically normal for larger subsample sizes and under weaker conditions than previously established.  Moreover, we provide Berry-Esseen bounds in order to quantify the rate at which this convergence occurs, making explicit the roles of the subsample size and the number of trees in determining the distribution of random forest predictions.  When these generalized estimators are reduced to their classical U-statistic form, the rates we establish are faster than any available in the existing literature.
	\end{abstract}
	
	
	\spacingset{1.5} 
	

	\section{Introduction}
	\label{introduction}
	
	The random forest algorithm is a supervised learning tool introduced by \cite{Breiman2001} that constructs many independently randomized decision trees and aggregates their predictions by averaging in the case of regression or taking a majority vote for classification. Random forests have been shown to successfully handle high-dimensional and correlated data while exhibiting appealing properties such as fast and accurate off-the-shelf fitting without the overfitting issues that often plague related methods. They have been successfully applied in a variety of scientific fields including remote sensing \citep{belgiu2016random}, computational biology \citep{qi2012random}, stock price forecasting \citep{khaidem2016predicting}, and forecasting bird migration \citep{coleman2020statistical}.  In a recent large-scale empirical study comparing 179 classifiers across the 121 datasets comprising the entire UCI machine learning repository, \cite{Fernandez2014} found that random forests performed extremely well with 3 of the top 5 algorithms being some variant of the standard procedure.
	
	Despite their wide-ranging applicability and well-documented history of empirical success, establishing formal mathematical and statistical properties for random forests has proved quite difficult, due in large part to the complex, data-dependent nature of the CART-splitting criterion \citep{CART} traditionally used to construct individual trees. \cite{Breiman2001} provided the first such result, demonstrating that the expected misclassification rate is a function of the accuracy of the individual classifiers and the correlation between them. The bound on the misclassification rate postulated in the work is loose but suggestive in the sense that interplay between these two sets the foundation for understanding the inner-workings of the procedure.  \cite{banerjee2007} established a limit law for the split location in a regression tree context with independent Gaussian noise. Further analysis of the behavior of CART-style splitting was conducted by \cite{Ishwaran2015} who demonstrated an end-cut preference, whereby splits along non-informative variables are more likely to occur near the edges of the feature space.  
	
	A variety of other work has focused on analyzing other properties of random forest ensembles or extending the methodology to related problem types.  \cite{Lin2006} developed the idea of potential nearest neighbors and demonstrated their relationship to tree-based ensembles. More recently \cite{Lopes2020} analyzed the tradeoff between the size of the ensemble and the classification accuracy.   \cite{Biau2008}, \cite{Biau2010}, and \cite{Biau2012} studied various idealized versions of random forests and investigated consistency while \cite{pmlr-v28-denil13} proved consistency for a particular type of online forest.  \cite{Ishwaran2008} developed the idea of random survival forests and the consistency of such models is investigated in \cite{Ishwaran2010} and \cite{Cui2019}.  \cite{Meinshausen2006} extended random forest estimates to the context of quantile regression and \cite{Zhu2015} experimented with reinforcement learning trees.  For a more detailed accounting of random-forest-related research, we refer readers to an excellent recent review by \cite{Biau2016}. 
	
	In recent years, many promising developments have come by considering individual trees built with subsamples rather than the more traditional bootstrap samples.  \cite{Wager2014} extended the infinitesimal jackknife estimates of variance introduced by \cite{Efron2014} to produce confidence intervals for subsampled random forest predictions.  \cite{Scornet2015} provided the first consistency result for Breiman's original forests, establishing $L^2$ consistency whenever the underlying regression function is additive.  \cite{Mentch2016} made the connection to infinite-order U-statistics and provided the asymptotic distributions of random forest predictions. \cite{Wager2018} showed that for large ensembles, subsampled random forests are both asymptotically unbiased and Gaussian whenever individual trees are built according to honesty and regularity conditions.  
	
	In this paper, we continue the trend of establishing mathematical properties of random forests by building on the U-statistic connection made in \cite{Mentch2016}.  As in other recent theoretical analyses on the topic (e.g.\ \cite{Biau2008,Biau2010,Biau2012,Mentch2016,Wager2018}), we adopt a general notion of random forests, viewing this class of estimators as those producing predictions of the form
	\begin{equation*}
		\label{eqn:rfDef}
		\text{RF}(x) = \frac{1}{N} \sum_{i=1}^{N} h(x; Z_{i1}, \dots , Z_{is}; \omega) 
	\end{equation*}
	where each $Z_{i1}, \dots , Z_{is}$ denotes a subsample taken without replacement from the available training data and $\omega$ denotes additional randomness injected into the base learner $h$.  In particular, we do not require that base learners be trees constructed via the CART methodology as originally proposed in \cite{Breiman2001}.  We establish central limit theorems for such estimators, that, to our knowledge, cover a broader range of estimators and also allow for faster subsampling rates than established in existing literature.  More notably, we take a step forward in the theoretical analysis of random forests by providing Berry-Esseen Theorems governing the rate at which this convergence takes place by bounding the maximal error of approximation between the Gaussian distribution and that of the random forest predictions.  In establishing these, we develop the notion of \emph{generalized} U-statistics which allow for kernels to be incomplete, randomized, and infinite-order.  Importantly, when these estimators are simplified to their classical U-statistic form, the resulting bounds we provide are faster than any in the existing literature.
	
	The remainder of this paper is organized as follows. In \cref{background}, we provide additional background on the random forest algorithm and introduce the notion of generalized U-statistics. In \cref{normality} we provide a theorem that describes the asymptotic distribution of these statistics when the rank of the kernel is allowed to grow with $n$.  These distributional results rely on the behavior of a variance ratio and we conclude \cref{normality} by discussing its behavior for a variety of base learners.  Building on these preliminary results, in \cref{berry-esseen}, we provide Berry-Esseen bounds for both complete and incomplete generalized U-statistics.

	
	\section{Background}
	\label{background}
	
	Suppose that we have data of the form $Z_1,\dots , Z_n $ assumed to be independent and identically distributed (i.i.d.)\ from some distribution $F_{Z}$ and let $\theta$ be some parameter of interest. Suppose further that there exists an unbiased estimator $h$ of $\theta$ that is a function of $s \leq n$ arguments and without loss of generality, assume that $h$ is permutation symmetric in those arguments. The minimum variance unbiased estimator for $\theta$ given by 
	\begin{equation}
		\label{eq:u}
		U_{n,s} = {n\choose s}^{-1} \sum_{(n,s)} h(Z_{i1}, \dots , Z_{is})
	\end{equation}
	is a U-statistic as introduced by \cite{Halmos1946} and \cite{Hoeffding1948}, where the sum is taken over all $\binom{n}{s}$ subsamples of size $s$; we use the $(n,s)$ shorthand for this quantity throughout the remainder of this paper. Standard elementary examples of U-statistics include sample mean, sample variance and covariance, and Kendall's $\tau$-statistic. When both the kernel $h$ and rank $s$ are held fixed, \cite{Hoeffding1948} showed that $U_{n,s}$ tends toward a normal distribution with mean $\theta$ and variance ${s^2\zeta_1}/{n}$ where, for any $1 \leq c \leq s$, 
	\begin{equation*}
		\zeta_{c} = \Cov \left(h(Z_1,\dots,Z_c,Z_{c+1},\dots, Z_s),h(Z_1,\dots,Z_c,Z_{c+1}',\dots, Z_s')\right)
	\end{equation*}
	where $Z_{c+1}',\dots,Z_{n}'$ are i.i.d.\ from $F_{Z}$ and independent of $Z_1,\dots,Z_n$. 
	
	Throughout the remainder of this paper, we consider a regression framework where the data consist of independent pairs of random variables consisting of covariates and a response $Z_i=(X_i,Y_i)\in \mathcal{X}\times \R~(i = 1,\dots, n)$ sampled from a common distribution $F_Z$. Unless otherwise stated, we assume $\mathcal{X}=\R^p$ for analytical convenience.
	
	Given some $s \leq n$, let $Z_{i1}, \dots, Z_{is}$ denote a subsample of size $s$ and consider a particular location $x \in \R^p$.  The prediction at $x$ can be written as $h_x(Z_{i1},\dots , Z_{is})$ where the function $h_x$ takes the subsampled covariates and responses as inputs, forms a regression estimate, and outputs the predicted response at $x$.  Throughout the remainder of this paper, we drop the subscript $x$ for notational convenience.  Repeating this process on $N$ subsamples and averaging across predictions gives
	\begin{equation*}
		U_{n,s,N}(x) = \frac{1}{N} \sum_{i=1}^{N} h(Z_{i1}, \dots, Z_{is})
	\end{equation*}
	so that our prediction now takes the form of a U-statistic with kernel $h$.  When all subsamples are used so that $N = {n\choose s}$, the form is that of a \emph{complete} U-statistic; whenever a smaller number of subsamples are utilized, it is \emph{incomplete}.  When the subsample size $s$ grows with the sample size $n$, these estimators are considered \emph{infinite-order} U-statistics as introduced by \cite{Frees1989} and utilized by \cite{Mentch2016} to establish asymptotic normality of random forests.

	In a general supervised learning framework, these kernels can be thought of as base learners in an ensemble.  Decision trees are among the most popular choices of base learners and are typically built according to the CART criterion.  Here, splits in each cell $A$ are chosen to maximize
	\[\begin{aligned}
		L(j,z) &= \frac{1}{|A|}\sum_{i=1}^n(Y_i-\bar{Y}_A)^2\one_{X_{i}\in A} \\
		&\hspace{5mm} - \frac{1}{|A|}\sum_{i=1}^n(Y_i-\bar{Y}_{A_L}\one_{X_{j,i}<z} -\bar{Y}_{A_R}1_{X_{j,i}\geq z})^2\one_{X_i\in A}
	\end{aligned}\]
	across all covariates $X_j$, $1\leq j\leq p$, where $z\in \R$, $A_L = \{X\in A: X_{j}<z\}$, $A_R= \{X\in A: X_j\geq z\}$, and for any set $S$, $\bar{Y}_{S}$ denotes the average response value for observations $X \in S$.  When trees are built with bootstrap samples, the resulting ensembles produce \emph{bagged} estimates as discussed in \cite{Breiman1996}.  The random forest extension of bagging introduced by \cite{Breiman2001} inserts additional independent randomness into each tree, typically to determine the subset of $\text{\texttt{mtry}} \leq p$ features eligible for splitting at each node.  The subsampled version of this procedure thus produces estimates at $x$ of the form
	\begin{equation}
		\label{eqn:rf}
		\tilde{U}_{n,s,N,\omega}(x) = \frac{1}{N} \sum_{i=1}^{N} h(Z_{i1}, \dots, Z_{is}; \omega) .
	\end{equation}
	Note that for each decision tree we consider an i.i.d.\ sample of randomness $\omega_i$ but for notational convenience, we refer to this as simply $\omega$ for all trees.  Furthermore, in a similar fashion as above, define $\zeta_{c,\omega}~(c = 1,\ldots, s-1)$ and $\zeta_s$ as 
	\begin{equation}
		\label{zeta}
		\begin{aligned}
			&\zeta_{c,\omega} = \Cov \left(h(\dots,Z_c,Z_{c+1},\dots, Z_s; \omega),h(\dots,Z_c,Z_{c+1}',\dots, Z_s'; \omega')\right) \\
			&\zeta_s =  \Cov \left(h(\dots,Z_c,Z_{c+1},\dots, Z_s; \omega),h(\dots,Z_c,Z_{c+1},\dots, Z_s;  \omega)\right)
		\end{aligned}
	\end{equation}
	and note that $\zeta_s$ is simply the variance of the kernel with randomization parameter $\omega$.

	\cite{Mentch2016} provide asymptotic distributional results for $\tilde{U}_{n,s,N,\omega}$ with respect to their individual means that cover all possible growth rates of $N$ with respect to $n$, though the form of the result provided has several practical limitations.  In particular, the authors require that $\zeta_{1,\omega}$ does not approach $0$, but for most practical base learners, the correlation between estimators with only one observation in common should vanish as the subsample size grows.  Indeed, Lemma \ref{lemma1} in \cref{app:background} gives that $\zeta_{1,\omega} \leq \frac{1}{s}\zeta_{s,\omega}\leq \frac{1}{s}\zeta_s$ so that when $\zeta_s$ is bounded, $\zeta_{1,\omega} \to 0$ as $s \to \infty$.  In very recent work, \cite{Romano2019} showed that the same result could be obtained under a more mild condition.  In both results, however, the subsample size is limited to $s = o(n^{1/2})$ which can be quite restrictive in practice.  In Appendix \ref{app:background}, we demonstrate that this limitation is a result of a reliance on H\'ajek projections and in fact, whenever such an approach is taken, there is strong reason to believe that a subsampling rate of $s=o(n^{1/2})$ is the largest possible.  As later discussed by \cite{Wager2018} however, when $s$ is small, it is possible that the squared bias decays slower than the variance, thereby producing confidence intervals which, when built according to the stated Gaussian limit distribution, may not cover the true value. \cite{Wager2018} provide an alternative central limit theorem for averages over trees built according to honesty and regularity conditions.  When base learners conform to such conditions and $N$ is very large, the authors show that the subsampling rate can be improved to $s = o(n^{\beta})$ for $0.5 < \beta < 1$ while retaining consistent estimates.

	Motivated by the form of \eqref{eqn:rf} we now formalize the notion of \emph{generalized} U-statistics. 
	\begin{definition}[generalized U-statistic]
		\label{gu}
		Suppose $Z_1, \dots, Z_n$ are i.i.d.\ samples from $F_Z$ and let $h$ denote a (possibly randomized) real-valued function utilizing $s$ of these samples that is permutation symmetric in those $s$ arguments. A generalized U-statistic with kernel $h$ of order (rank) $s$ refers to any estimator of the form 
		\begin{equation}
			\label{eqn:gu}
			U_{n,s,N,\omega} = \frac{1}{\hat{N}} \sum_{(n,s)} \rho h(Z_{i1}, \dots, Z_{is}; \omega)
		\end{equation}
		where $\omega$ denotes i.i.d.\ randomness, independent of the original data. 
		The $\rho$ denotes i.i.d.\ Bernoulli random variables determining which subsamples are selected and $\P(\rho=1) = N/{n\choose s}$. The actual number of subsamples selected is given by $\hat{N} = \sum \rho$ where $\E[\hat{N}] = N$.
		When $N={n\choose s}$, the estimator in \cref{eqn:gu} is a generalized complete U-statistic and is denoted as $U_{n,s,\omega}$.  When $N < {n\choose s}$, these estimators are generalized incomplete U-statistics. 
	\end{definition}
	
	Though it is not practical to simulate ${n\choose s}$ Bernoulli random variables, fortunately, it is equivalent to first simulate $\hat{N}\sim \mathrm{Binomial}({n\choose s}, N/{n\choose s})$ and then randomly generate $\hat{N}$ subsamples without replacement. Note also that while the number of subsamples $\hat{N}$ in \cref{eqn:gu} is random, it concentrates around $N$.  
	
	Allowing for the possibility of a randomized kernel is of benefit here as it allows the results that follow to pertain to the kinds of randomized ensembles often considered in practice. The randomization parameter $\omega$ might, for example, perform some kind of feature subsampling as is commonly associated with random forests -- much further discussion along these lines is provided in \cref{normality}.  We stress however that the mere inclusion of such a randomization parameter is not where the true innovation in our work lies, nor should it be viewed as the ``essential ingredient" in what we refer to as generalized U-statistics.  Indeed, in several of the results that follow, the theoretical details needed to establish them follow a near-identical recipe regardless of whether the kernel itself takes on additional randomness.  
	
	Rather, the real benefit of considering generalized U-statistics lies in the form of the estimator itself that allows for, in essence, a random weighting to be applied to the kernel through the use of $\rho$.  
	As a bit of a preview of what is to follow, note also that in this generalized form, an incomplete U-statistic can be viewed as merely a complete U-statistic with a different kernel up to $N/\hat{N}$.  It is these kinds of realizations that provide significant benefits for theoretical analysis by allowing us to view incomplete U-statistics as merely a modified version of its complete form, rather than as an approximation to it that inherits a remainder term that needs to be controlled.  Furthermore, in the complete case, it can be shown that the variance of the U-statistic can be decomposed into a sum over $s$ terms and that the structure of the statistic itself shrinks the higher-order terms in that sum.  This careful examination of higher-order terms allows us to not only establish asymptotic normality, but to provide rates of convergence sharper than any in the existing literature, some of which are based on fundamental work dating back several decades.

	In the literature on classic U-statistics, many results are derived by applying a technique called the H-decomposition, which allows the statistic to be written as a sum of uncorrelated terms. \cref{app:B-HD} contains a detailed overview of the classic H-decomposition.  The idea was first introduced by \cite{Hoeffding1961}, but has analogues in many parts of statistics, most notably in the analysis of variance in balanced experimental designs; for a more general result, see \cite{Efron1981}. To handle \emph{generalized} U-statistics, we begin by extending the concept of the H-decomposition to this more general setting.
	
	\begin{definition}[H-decomposition]
		Suppose that $Z_1,\dots, Z_s$ are i.i.d.\ samples from $F_Z$ and  $h(z_1,\ldots, z_s; \omega)$ is a (possibly randomized) real valued function that is permutation-symmetric in $(z_1,\dots, z_s)$.  Let $h_i(z_1,\ldots, z_i) = \E[h(z_1,\ldots, z_i,Z_{i+1},\dots, Z_s;\omega)]-\E[h]$ for $ i=1,\dots, s$ and let
		\begin{align*}
			h^{(i)} &= h_i(z_1,\dots, z_i)-\sum_{j=1}^i\sum_{(s,j)}h^{(j)}(z_{i1},\dots, z_{ij}), \quad  \text{ for } i=1,\dots, s-1 \text{ and } \\
			h^{(s)} & = h(z_1,\dots, z_s;\omega) -\sum \limits_{j=1}^{s-1}\sum\limits_{(s,j)}h^{(j)}(z_{i1},\dots, z_{ij}).
		\end{align*}
		The H-decomposition of a generalized complete U-statistic is expressed as
		\begin{equation}
			\label{gu-hd}
			U_{n,s,\omega} = \sum_{j=1}^s{s\choose j}{n\choose j}^{-1}  \sum_{(n,j)}h^{(j)}(Z_{i1},\dots, Z_{is}).
		\end{equation}
	\end{definition}
	When no extra randomness is injected into $h$, the above definition reduces to the classic H-decomposition. Note that the randomness $\omega$ is only involved in $h^{(s)}$; for $h^{(1)},\dots, h^{(s-1)}$, it is marginalized out. Note that because each subsample is associated with an i.i.d.\ draw of the randomness $\omega$, each of the $h^{(s)}$ terms in \cref{gu-hd} involve this randomness though this notation is suppressed in \cref{gu-hd} for readability.

	
	\section{Asymptotic Normality}
	\label{normality}
	
	Before providing the asymptotic distributional results for generalized U-statistics of the form
	\begin{equation}
		\label{eqn:GenUstat}
		U_{n,s,N,\omega} = \frac{1}{\hat{N}} \sum_{(n,s)} \rho h(Z_{i1}, \dots, Z_{is}; \omega)
	\end{equation}
	we pause to emphasize the value in considering this form of estimator and to distinguish this generalization from the more classical counterparts considered in recent studies.  \cite{Mentch2016} produce a central limit theorem for infinite-order U-statistics, but consider randomized kernels only insofar as establishing that when such randomness is well-behaved, the limiting distributions are equivalent.  More recently, \cite{Wager2018} analyzed random forests constructed with all possible subsamples where the kernel can thus be written in a form where the additional randomness is marginalized out.  Such estimators take the form
	\begin{equation}
		\label{eqn:WA18}
		\binom{n}{s}^{-1} \sum_{(n,s)} {\E}_{\omega} h(Z_{i1}, \dots, Z_{is}; \omega)
	\end{equation}
	so that the kernels themselves are non-random and thus the estimator is simply a complete, infinite-order U-statistic with kernel $g = \mathbb{E}_{\omega} h(Z_{i1}, \dots, Z_{is}; \omega)$.
	
	While more convenient for theoretical analysis, random forests of the form conceived in \eqref{eqn:WA18} are not generally utilized in practice, even in small-data settings since, by construction, such a statistic involves building every possible randomized tree on every possible subsample of the data.  In practice, random forests might be loosely seen as a \emph{double} or \emph{nested} Monte Carlo approximation to the estimators in \eqref{eqn:WA18}, where one source of approximation results from using $N < {n\choose s}$ subsamples and the other results from estimating the kernel itself $\mathbb{E}_{\omega} h(Z_{i1}, \dots, Z_{is}; \omega)$ on each subsample. Recent work by \cite{rainforth2018nesting} provides an analysis of these kinds of nested approximations.
	
	In practice, however, random forests are nearly always constructed by selecting subsamples at random and pairing each with an independently selected randomization instance $\omega$, which is itself generally assumed to be selected uniformly at random.  Generalized U-statistics therefore provide a direct and accurate representation of such estimators.  We begin with a theorem establishing asymptotic normality for complete generalized U-statistics. 
	
	\begin{theorem}\label{AN-u}
		
		Let $Z_1,\dots, Z_n$ be i.i.d.\ from $F_Z$ and  $U_{n,s,\omega}$ be a generalized complete U-statistic with kernel $h(Z_1,\dots,Z_{s};\omega)$. Let $\theta=\E[h]$, $\zeta_{1,\omega}=\V(\E[h|Z_1])$ and $\zeta_s=\V(h)$.
		If $\frac{s}{n}\frac{\zeta_s}{s\zeta_{1,\omega}}\to 0 $, then
		\begin{equation}
			\label{eq:M1}
			\frac{U_{n,s,\omega}-\theta}{\sqrt{s^2\zeta_{1,\omega}/n}} \rightsquigarrow N(0,1).
		\end{equation}
		
	\end{theorem}
	The  proof of \cref{AN-u} is provided in the \cref{app:B-AN}. The general strategy is to find a linear statistic to approximate $U_{n,s,\omega}$, and show that the difference is negligible by applying the H-decomposition.

	\begin{remark} 
		The condition in \cref{AN-u} that $\frac{s}{n}\frac{\zeta_s}{s\zeta_{1,\omega}}\to 0$ can be replaced by the weaker condition that $\frac{s}{n}(\frac{\zeta_s}{s\zeta_{1,\omega}}-1) \to 0$.  In practice, this condition can be satisfied by choosing $s$ to grow slow relative to the variance ratio ${\zeta_s}/{s\zeta_{1,\omega}}$.  In particular, whenever the ratio is bounded, choosing $s=o(n)$ is sufficient.  Thus, in establishing asymptotic normality, this weaker condition may be of minimal consequence.  However, in quantifying the finite sample deviations from normality via the Berry-Esseen Theorems in \cref{berry-esseen}, this alternative condition plays an important role in establishing the bounds provided.
	\end{remark}

	Similar results for non-generalized U-statistics have appeared in the recent works discussed earlier.  Theorem 1 in \cite{Mentch2016} can be modified slightly to provide an analogous result whenever $\frac{s^2}{n}\frac{\zeta_s}{s\zeta_{1,\omega}}\to 0$. A recent result in \cite{Romano2019} proceeds along these lines.  Both results, however, could be improved by applying the H-decomposition rather than the H\'ajek projection.  Similarly, Theorem 3.1 in \cite{Wager2018} establishes asymptotic normality for non-generalized, complete U-statistics whenever the subsample size $s$ grows like $n^{\beta}$ for some $\beta < 1$.  Here though the authors are concerned only with base learners that take the form of averages over honest and regular trees and in particular, with controlling the asymptotic bias of the resulting estimator. Thus, with minor modifications, Theorem 3.1 in \cite{Wager2018} could be seen as something of a corollary to our \cref{AN-u} above, corresponding to the special case where the within-kernel randomness is held fixed or marginalized out. 
	
	The complete forms of these estimators are almost never utilized in practice due to the computational burden involved with calculating ${n\choose s}$ base learners.  Thus, armed with the results for the complete case, we now establish an analogous result for \emph{incomplete} generalized U-statistics. 
	
	\begin{theorem}\label{AN-iu}
		
		Let $Z_1,\dots,Z_n$ be i.i.d.\ from $F_Z$ and $U_{n,s,N,\omega}$ be a generalized incomplete U-statistic with kernel $h(Z_1,\dots,Z_{s};\omega)$. Let $\theta=\E[h]$, $\zeta_{1,\omega}=\V(\E[h|Z_1])$ and  $\zeta_s=\V(h)$. Suppose that $\E[|h-\theta|^{2k} ]/\E^2[|h-\theta|^k]$ is uniformly bounded for $k=2,3$ and for all $s$. If $\frac{s}{n} \frac{\zeta_s}{s\zeta_{1,\omega}} \to 0 $ and $N\to \infty$, then 
		\begin{equation}
			\label{eq:M2}
			\frac{U_{n,s,N,\omega}-\theta}{\sqrt{ {s^2\zeta_{1,\omega}}/{n} + {\zeta_{s}}/{N}}} \rightsquigarrow N(0,1).
		\end{equation}
	\end{theorem}
	
	\begin{remark} 
		Note that the variance in the theorem above takes a different form than in \cref{AN-u} but the requirement that $\frac{s}{n}\frac{\zeta_s}{s\zeta_{1,\omega}}\to 0$ remains the same.  Indeed, whenever this condition is satisfied, the complete U-statistic analogue is also asymptotically normal and normality of the incomplete version can be established as a by-product.  However, this condition and more generally, asymptotic normality of the complete version, is not necessary.  In such cases, choosing a very small ensemble size (e.g.\ N=o(n/s)) is sufficient.  More details and related results are provided in \cref{app:B-AN} along with the proof of \cref{AN-iu}.
	\end{remark}
	
	Taken together, Theorems \ref{AN-u} and \ref{AN-iu} provide the asymptotic distribution of generalized U-statistics for all possible growth rates on the number of subsamples $N$ relative to $n$.  Besides the regularity conditions on the kernel, these results require only that $\frac{s}{n}\frac{\zeta_s}{s\zeta_{1,\omega} \to 0} $.  This condition, similar to the notion of $\nu$-incrementality discussed in \cite{Wager2018}, is not overly strong but may appear somewhat arbitrary.  In the following subsection we investigate the behavior of this ratio for a variety of base learners.

	\subsection{Variance Ratio Behavior} \label{var_ratio}
	
	For a given kernel $h$, let $\hat{h}$ be the projection of $h$ onto the linear space.  We have that $\hat{h}= \sum_{i=1}^s h_1(Z_i)$ and thus
	\begin{equation}
		\label{vr}
		\frac{\V(h)}{\V(\hat{h})} = \frac{\V(h)}{s\V(\E[h|Z_1])} = \frac{\zeta_s}{s \zeta_{1,\omega}}  .
	\end{equation}
	Since $\zeta_s$ is the overall variance and $\zeta_{1,\omega}$ can be written as the variance of the expectation of the kernel conditioning on one argument, we can view the ratio in \cref{vr} as a measure of the potential influence of one single observation on the output of the kernel.  When $\zeta_s/s\zeta_{1,\omega} \to 1$, $h$ itself is asymptotically linear. More generally though, Theorems \ref{AN-u} and \ref{AN-iu} require only that $\frac{s}{n}\frac{\zeta_s}{s\zeta_{1,\omega}} \to 0 $ in order for the generalized U-statistic to be asymptotically normal.  Thus, if the limiting behavior of the variance ratio ${\zeta_s}/{s \zeta_{1,\omega}}$ is understood, the subsampling rate can be chosen to ensure the entire term approaches 0.  
	
	For simple kernels such as the sample mean and sample variance, it is straightforward to show that the limit of this variance ratio is 1, though this can also be shown to hold for more standard regression estimates such as ordinary least squares; see \cref{app:B-VR} for explicit calculations.  Here we focus our attention more on nearest-neighbor estimators and linear smoothers as these are more directly relatable to the tree-style base learners often used in practice.
	
	\begin{proposition}
		\label{knn}
		
		Let $Z_1,\dots, Z_s$ denote i.i.d.\ pairs of random variables $(X_i,Y_i)$ and suppose $Y_i = f(X_i) + \epsilon_i$ where $f$ is continuous, $\epsilon_i$ has mean $0$ and variance $\sigma^2$, and $X_i$ and $\epsilon_i$ are independent. Let $\varphi$ denote the standard k-nearest neighbor (kNN) estimator.  Then 
		\[
		\limsup_{s\to \infty}\frac{\zeta_s}{s\zeta_1} \leq c(k)
		\] 
		where 
		\[c(k) =2k /\left[ \sum_{i=0}^{k-1}\sum_{j=0}^{k-1} \frac{(i+j)!}{i!j!}\frac{1}{2^{i+j}}\right]
		\]
		so that $c(k)$ is decreasing in $k$ and $1 < c(k) \leq 2$. 
		
	\end{proposition}
	
	\begin{figure}[t]
		\centering
		\includegraphics[width=0.8\textwidth]{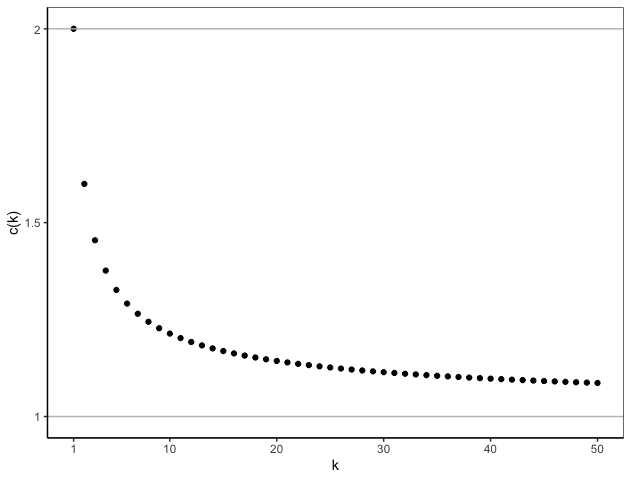}
		\caption{A plot of c(k) for $k = 1,..., 50$.}
		\label{fig:ck}
	\end{figure} 
	
	A sketch of $c(k)$ for $k = 1,..., 50$ is shown in \cref{fig:ck}. The proof of \cref{knn} is provided in \cref{app:B-VR}. Note that kNN is a nonadaptive linear smoother, the variance ratio of which is bounded above by a constant. The following result gives an upper bound for the more general class of all linear smoothers.
	\begin{proposition}
		\label{ls}
		
		Let $Z_1,\dots , Z_s$ denote i.i.d.\ pairs of random variables $(X_i,Y_i)$ and suppose $Y_i = f(X_i)+\epsilon_i$ where $f$ is bounded, $\epsilon$ has mean $0$ and variance $\sigma^2$, and $X_i$ and $\epsilon_i$ are independent. Let 
		$\varphi = \sum_{i=1}^s w(i,x,\mathbf{X})Y_i$
		such that $\sum_{i=1}^s w(i,x,\mathbf{X}) =1$, where $\mathbf{X}$ denotes $\{X_i\}_{i=1}^s$.  Then 
		\begin{equation*}
			\limsup_{s\to \infty }\frac{\zeta_s/{s\zeta_1}}{s} <\infty.
		\end{equation*}
	\end{proposition}
	
	The results above demonstrate that the behavior of the variance ratio is manageable for k-nearest neighbor base learners and more generally, linear smoothers.  Recent work \citep{Scornet2016kernel,Olson2018} has sought to draw a connection between these estimators and the CART-style trees utilized in Breiman's original random forests. The purely random forest \citep{Biau2008} that determines splits completely at random, for example, is exactly a linear smoother and thus by the above result, has a variance ratio that behaves like $O(s)$.  In work dating back even further, \cite{Lin2006} introduced the concept of \emph{potential} nearest neighbors (PNNs) and showed that random forests can be viewed as an adaptively weighted $k$-PNN method. 
	
	\begin{definition}[\citep{Lin2006}]
		A sample point $Z_i=(X_i,Y_i)$ is called a k-potential nearest neighbor ($k$-$\mathrm{PNN}$) of a target point $x$ if and only if there are fewer than $k$ sample points other than $X_i$ in the hyperrectangle defined by $x$ and $X_i$. 
	\end{definition}
	
	Typically, the number of potential nearest neighbors is much larger than the number of nearest neighbors. Existing nearest-neighbor methods, both adaptive and nonadaptive, predict by selecting and averaging over $k$ points from the set of all $k$-$\mathrm{PNNs}$. The classical kNN procedure non-adaptively chooses the $k$ points as those closest to $x$ under some metric whereas commonly used tree-based methods may have a terminal size bounded by $k$ and adaptively select points from the $k$-$\mathrm{PNNs}$ based on empirical relationships in the data. 
	
	Moving closer to this, consider the base learner that forms a prediction at $x$ by simply choosing $k$ of the $s$ observations in the subsample uniformly at random and averaging the corresponding response values.  In  \cref{app:B-RP}, we show that the resulting variance ratio for this naive estimator is given by
	\[
	\frac{\zeta_s}{s \zeta_{1,\omega}} = \frac{s}{k} = \frac{1/k\cdot s^2}{s\cdot 1}.
	\]
	Now reconsider the kNN base learner. We can view such an estimator as ``randomly" selecting $k$ points from the $k$-$\mathrm{NN}$s and again predicting by taking the average. In this case, the variance ratio can be written as 
	\begin{equation*}
		\frac{\zeta_s}{s\zeta_{1,\omega}} = O(1) = O\left(\frac{1/k\cdot k^2}{s\cdot\E[\P^2(X_1\in \mathrm{kNN}\mid X_1)]}\right).
	\end{equation*}
	The form of this result may naturally lead one to conjecture that for any base learner that predicts by randomly selecting and averaging over points in some set $A$, the resulting ratio may have the form
	\begin{equation}
		\label{conj}
		\frac{\zeta_s}{s \zeta_{1,\omega}} = O\left(\frac{1/k\cdot |A|^2}{s \cdot \E[\P^2(X_1\in A \mid X_1)]} \right).
	\end{equation}
	Consider then a simple tree-style estimator that predicts at $x$ by sampling $k$ points uniformly at random from its $k$-PNNs and averaging the corresponding response values; we refer to these random potential nearest neighbor estimators as RP trees.  The additional difficulty introduced with RP trees is that the size of this set of potential nearest neighbors is itself random, though from \cite{Lin2006}, we know that the expected number of $k$-PNNs is $O(k(\log s )^{p-1})$ and so extending our conjecture, we arrive at the following proposition. 
	We have 
	\begin{proposition}
		\label{rptree}
		Let $Z_1,\dots , Z_s$ denote i.i.d.\ pairs of random variables $(X_i,Y_i)$ and suppose $Y_i = f(X_i)+\epsilon_i$ where $f$ is bounded, $\epsilon$ has mean $0$ and variance $\sigma^2$, and $X_i$ and $\epsilon_i$ are independent. Suppose further that the density of $X$ is bounded away from 0 and infinity in $[0,1]^p$. Then for the RP tree estimator, we have 
		\begin{equation}
			\limsup_{s\to \infty}\frac{\zeta_s/{s \zeta_{1,\omega}}}{(\log s)^{2p-2}} <\infty.
		\end{equation}
	\end{proposition}
	The proof of  \cref{rptree} is provided in \cref{app:B-RP}. Here, asymptotic normality can be ensured by insisting on the same subsample sizes put forth in \cite{Wager2018}, namely that $s = o(n^{\beta})$ for some $0.5 < \beta < 1$.
	
	To bridge the gap even further between these completely non-adaptive trees that decide splits independent of any data and the completely adaptive CART-style trees often employed in practice that utilize all available data to decide both the tree structure and final predictions, we turn now to the idea of \emph{honesty} and \emph{double sample} trees.
	\begin{definition}
		We call a tree-style predictor a non-adaptive $k$-tree if $\mathrm{(a)}$ the splitting does not depend on the response values of the data, but could potentially depends the features and some other random mechanism that is independent of the data, and $\mathrm{(b)}$ the terminal size is bounded between $k$ and $2k -1$. 
	\end{definition}
	Note that part (a) of the above definition is what was referred to as the "honesty" condition in \cite{Wager2018}.  
	According to this definition, RP trees are simply a special type of non-adaptive $k$-tree. We show these RP trees are a worst case non-adaptive $k$-tree in the sense that the $\zeta_1$ parameter resulting from a non-adaptive $k$-tree can be lower bounded by that of an RP tree. 
	\begin{proposition}
		\label{non-adaptive}
		Under the same conditions on the underlying distribution as in \cref{rptree}, for the non-adaptive $k$-PNN tree, we have 
		\begin{equation}
			\limsup_{s\to \infty}\frac{\zeta_s/{s \zeta_{1,\omega}}}{(\log s)^{2p-2}}<\infty.
		\end{equation}
	\end{proposition}
	
	The proof of \cref{non-adaptive} is provided in \cref{app: B-AD}.  A similar result was obtained in \cite{Wager2018} where the denominator was given as $(\log s)^{p}$. \cref{app: B-AD} provides an alternative proof. 
	
	As is becoming evident, we now find ourselves caught in a rut familiar to most most researchers studying the theoretical properties of statistical and machine learning procedures.  On one hand, the non-adaptive nature of the methods described above helps a great deal in providing mathematical tractability -- in our case, by making it feasible to bound $\zeta_{1,\omega}$.  On the other, the strong performance of CART-style decision trees in practice is generally attributed, at least in large part, to their adaptive nature.  Double sample trees, introduced in \cite{Wager2018}, can be thought of as something of a compromise between these two extremes, perhaps even offering the best of both worlds in some settings.
	
	Double sample trees are constructed by first splitting the training samples into two halves - $\dataset_1$ and $\dataset_2$.  The first half is used to obtain the structure of the tree and the second half is used to form predictions by averaging across the response values in $\dataset_2$ from those points that fall in the same terminal node as the target point.  Here we also impose that the terminal nodes in the trees contain between $k$ and $2k-1$ observations from $\dataset_2$.
	
	The advantage of double sample trees is somewhat obvious:  the space can still be partitioned into small regions based on the CART criterion using data originating from the same underlying process, yet, because that data is independent of the second set of data used to form predictions, we can still guarantee control of $\zeta_{1,\omega}$.  On the other hand, because only half of the data is being used for each process, one may worry that such models would lose predictive accuracy.  Recent work on regularization in random forests, however, (see, e.g., \cite{JMLR:v21:19-905,mentch2020getting}) suggests that this could actually be potentially advantageous in low-signal settings.
	
	To see how control of $\zeta_{1,\omega}$ can be established, consider some $Z_1 \in \dataset_2$.  If we regard $\dataset_1$ as the source of the extra randomness, then the double sample tree can be viewed as a non-adaptive $k$-tree with respect to $\dataset_2$.  In particular, we can establish good control of $\zeta_{1,\omega}$ due to \cref{non-adaptive}, where $\zeta_{1,\omega}=O(1/ s(\log s )^{2p-2})$. Indeed, letting $h_1(Z_1) = \E[h(\dataset_1, \dataset_2;\omega)\mid Z_1]$, by law of total variation, we have
	\begin{equation}
		\begin{aligned}
			\zeta_{1,\omega} 
			& = \V\left(h_1(Z_1)\right) \\
			& 	\geq \E[\V(h_1(Z_1)\mid {Z_1\in \dataset_1})1_{Z_1\in \dataset_1}] +  \E[\V(h_1(Z_1)\mid {Z_1\in \dataset_2}) 1_{Z_1\in \dataset_2}] \\
			& \geq  \Pr(Z_1\in \dataset_2) \V(h_1(Z_1)\mid {Z_1\in \dataset_2}) \\
			& =  \frac{1}{2} \V(h_1(Z_1)\mid {Z_1\in \dataset_2}) \\
			& = O(1/(s(\log s)^{2p-2})).
		\end{aligned}
	\end{equation}
	
	Note that  \cite{Wager2018} obtained a similar result for double sample trees by imposing further regularity conditions on the splitting process.
	We do not require any such conditions here. \\

	In general, calculating the variance ratio for adaptive base learners without imposing specific constraints on the base learners and/or data is quite challenging.  However, we conclude our discussion here by noting that the previous calculations offer some encouragement.  Given the $k$-PNNs of some target point $x$ and considering estimators that predict by averaging over some subset of these, we showed that for non-adaptive estimators like kNN, the variance ratio is bounded. On the other hand, when the samples are selected uniformly at random from all $k$-PNNs, the variance ratio is on the order of $(\log s)^{2p-2}$.  Further and more generally, when the samples are selected in any non-adaptive fashion from all $k$-PNNs -- including via the use of an external, independent dataset -- the variance ratio remains upper bounded by $O((\log s)^{2p-2})$.
	
	Tree-based estimators, by definition, predict by averaging over subsets of potential nearest neighbors, though as discussed above, the particular fashion in which those neighbors are chosen in practice is often heavily data-dependent.  If, however, we are in a common regression setting where the response is regressed on covariates that contain some signal, then trees may heavily weight only a subset of the potential nearest neighbors, particularly in directions that can account for some of the variability in the response.  In such settings, the variance ratio may be approximately of the form in \eqref{conj} for some smaller set $S$ and therefore be smaller than that of RP trees.  In \cref{app:B-SI}, we provide a simple simulation study that suggests that the variance ratios resulting from ensembles are well-behaved in practice regardless of whether the base learners are simple non-adaptive nearest neighbor methods or CART-based decision trees.  We also provide several plots demonstrating the asymptotic normality of the resulting predictions.  
	
	We close this section by noting that in order for these asymptotic distributions to be used for inferential purposes, one must estimate the corresponding parameters.  In practice, estimating $\zeta_{1,\omega} = \V(h_1(z))$ is often quite challenging.  Since $h_1(Z_i)$ is generally unknown and we don't know the underlying distribution $F_Z$, we cannot generate new data and use a Monte Carlo method.  Rather, we can only sample only from the original data $(Z_1,\dots, Z_n)$, which inevitably introduces dependence among $\hat{h}_1(Z_1), \dots, \hat{h}_1(Z_n)$.  
	
	Recent work, however, showed that when the conditions for \cref{AN-u} hold, the proposed  ``pseudo" infinitesimal jackknife estimator (closely related to that introduced in \cite{Efron2014}) is consistent whenever $N = {n\choose s}$ \cite{peng2021bias}.  The insight is that although the dependence introduces bias, that bias will be dominated by $\zeta_{1,\omega}$ if the U-statistic $U_{n,s,\omega}$ is dominated by its linear term.  Even more importantly, this work also showed that so long as $N$ is large enough to ensure that ${n}/(Ns\zeta_{1,\omega})\to 0$, the finite sample version of pseudo infinitesimal jackknife estimator remains consistent.  Note that $s\zeta_{1,\omega}\leq \zeta_s$ and $\zeta_s$ could tend towards 0, so it is necessary that $N \gg n$.  One unfortunate practical takeaway is that the number of trees required for accurately estimating $\zeta_{1,\omega}$ can potentially be much larger than the number required to stabilize predictions from the random forest.  Interested readers are also invited to see recent work by \cite{zhou2019v} that explores this variance estimation problem through the lens of V-statistics where subsamples are drawn with replacement.  Note that if the asymptotic linearity condition in \cref{AN-u} does not hold, to the best of our knowledge, there are no methods in the existing literature capable of estimating $\zeta_{1,\omega}$ consistently given $s=o(n)$. This can be explained by the fact that since $\zeta_{1,\omega}$ represents the covariance between two trees sharing only a single common sample, whenever $s$ is large, we do not have sufficiently many such pairs with such little dependence.
	
	Finally, we note that estimating these variance parameters is not always essential for carrying out formal inference in a random forest context.  Recognizing the potential computational strain involved in estimating $\zeta_{1,\omega}$, \cite{coleman2019scalable} recently developed an asymptotically valid permutation test for assessing variable importance in random forests.  Here, rather than explicitly estimate the sampling distribution of predictions, the authors suggest building just one additional forest in which the effects of the variables under investigation are muted.  Predictions are made using the trees in each forest and the individual trees are then randomly permuted between forests and their respective accuracies are recalculated.  The intuition here is that if the features in question have no effect on predictive accuracy, forests containing trees that are unable to utilize those features should be as accurate as those with no such restrictions.

	
	\section{Berry-Esseen Bounds}
	\label{berry-esseen}
	
	Given i.i.d.\ random variables $Z_1, \dots, Z_n$ with mean $\mu$ and variance $\sigma^2$, the Berry-Esseen theorem \citep{Berry1941,Esseen1942} provides a classical result describing the rate of convergence of $S_n = \sum_{i=1}^n (Z_i-\mu)/\sigma\sqrt{n}$ to the normal distribution. It states that provided the third moment  $v_3 = \E|Z-\mu|^3$ is finite, 
	\[
	\sup_{z\in \R} |F_n(z)-\Phi(z)|\leq \frac{Cv_3}{\sigma^3\sqrt{n}}
	\]
	\noindent where $F_n$ is the distribution function of $S_n$, $\Phi$ is the distribution function of the standard normal, and $C$ is a constant independent of $n$ and the $Z_i$.  Several authors (e.g.\ \cite{Callaert1978, Chan1977, Grams1973,Chen2010}) have since contributed various iterations of Berry-Esseen theorems for U-statistics. In the following sections, we derive bounds for \emph{generalized} U-statistics involving $n,s,N$, and the moments of the base learner to lend some intuition regarding how these parameters might be chosen in practice.  We utilize the H-decomposition along with novel representations of U-statistics in order to provide bounds sharper than previously established in the literature for infinite-order U-statistics as well as first-of-their-kind bounds for \emph{generalized} U-statistics.

	\subsection{Bounds for Generalized  U-statistics}
	\label{berry-esseen-thms}
	
	We begin with the following result on generalized, complete U-statistics.
	\begin{theorem}
		\label{be-1}
		
		Suppose that $Z_1,\dots, Z_n$ are i.i.d.\ from  $F_Z$ and that $U_{n,s,\omega}$ is a generalized complete U-statistic with kernel $h=h(Z_1,\dots,Z_{s}; \omega)$. Let $\theta=\E[h]$, $\zeta_s=\V(h)$ and $\zeta_{1,\omega} = \E[g^2(Z_1)]$, where $g(z)= \E[h(z,Z_2,\dots,Z_{s}; \omega)]-\theta$. Suppose further that $\zeta_s<\infty$ and $\zeta_{1,\omega}>0$, then 
		\begin{equation*}
			\sup_{z\in \R}\left|\P\left\{\frac{U_{n,s,\omega}-\theta}{\sqrt{s^2\zeta_{1,\omega}/n}}\leq z\right\}- \Phi(z)\right|\leq \frac{6.1 \E|g|^3}{n^{1/2}\zeta_{1,\omega}^{3/2}} +(1+\sqrt{2}) \left\{\frac{s}{n} \left(\frac{\zeta_s}{s\zeta_{1,\omega}}-1\right)\right\}^{1/2}.
		\end{equation*}
		
	\end{theorem}
	
	A number of important points are worth noting here.  First, when $s$ is fixed, the bound has a rate on the order of $1/\sqrt{n}$ as should be expected since this is the standard rate associated with classic (finite-order), complete U-statistics.  Additionally, when the randomness $\omega$ is held fixed so that the estimator reduces to an infinite-order U-statistic, this bound is sharper than that provided in \cite{Chen2010}, which, to our knowledge, is the sharpest to date in the existing literature.  Specifically, the bound appearing in \cite{Chen2010} replaces the term 
	\[
	\frac{s}{n} \left(\frac{\zeta_s}{s\zeta_1}-1\right)
	\]
	in the bound above, which is on the order of $s/n$, with 
	\[
	\frac{(s-1)^2\zeta_s}{s(n-s+1)\zeta_1}
	\]
	which is on the order of $s^2/n$.  An immediate consequence of this tighter bound is that when kernels are employed such that the resulting terms ${\zeta_s}/{s\zeta_1}$ and ${\E|g|^{3}}/{\zeta_1^{3/2}}$ are bounded, a subsampling rate of $s=o(n)$ is sufficient to ensure the bound converges to 0, whereas a rate of $s=o(\sqrt{n})$ would be required according to the bound given in \cite{Chen2010}.  This sharper rate we obtain is ultimately a result of utilizing \cref{chenCol} together with the H-decomposition. Full details are provided in \cref{app:C-BR}.
	
	We note also that recent work by \cite{Song_2019} considered a similar high-dimensional statistic and obtain the Berry-Esseen bound
	\begin{equation}
		\label{song}
		O\left(\frac{D_n^2\log^{q_1}(dn)}{\underline{\sigma}^2_gn}\right)^{1/6} +  O\left(\frac{s^3}{n}\left[\frac{\log^{q_2}(d)D_n^2}{s\underline{\sigma_g}^2}\right]\right)^{1/4}
	\end{equation}
	where, $q_1, q_2$ are constants, $d$ is the dimension of the kernel, and $D_n$ is a uniform bound on the Orlicz norm of the components of the kernel. Note that the first term in \cref{song} comes from the Berry-Esseen bound for the H\'ajek projection (which is a high dimensional sample mean; see \cite{chernozhukov2017central}) and the second term comes from controlling the remainder of that H\'ajek projection. The two terms in this sum correspond directly to the two terms in \cref{be-1}.

	Comparing the two bounds, we see that the second term is raised to the power of $1/4$ in \cref{song} but to the power of $1/2$ in our result in \cref{be-1}.  The authors in \cite{Song_2019} also argue that the term $\underline{\sigma_g}^2$ is lower bounded by $O(s^{-2})$, ultimately implying that $s$ is required to be at most $n^{1/4-\epsilon}$,  $\forall \epsilon>0$, whereas we require only that $s=O(n^{1-\epsilon})$.  It seems to be the \emph{multi}-dimensionality alone rather than the \emph{high}-dimensionality that is driving the suboptimal rates in \cref{song}, because we do not have a sharp bound on the size of the remainder term nor a sharp concentration inequality like that in \cref{RCE} to apply in higher dimensions.  The authors consider bounding $\underline{\sigma}^2_g$, where $\underline{\sigma}_g^2=\min \zeta_{1,j}$ for $j=1,\dots, d$, via the Cram\'er-Rao inequality and end up with $\zeta_{1,j} = O(s^{-2})$. This equality may be conservative since the bound considers the worst-case scenario. In most of our examples, $\zeta_{1}$ can be on the order of  $s^{-1-\epsilon}$, $\forall \epsilon>0$.  \\

	
	Generalized incomplete U-statistics can be viewed as generalized \emph{complete} U-statistics with an alternative kernel up to the scalar $N/\hat{N}$. Recognizing this fact, we can make use of the H-decomposition and Lemma \ref{chenCol} given in the appendix to obtain the following bound for incomplete, generalized U-statistics. 
	
	\begin{theorem}
		\label{be-2}
		
		Suppose that $Z_1,\dots, Z_n$ are i.i.d.\ from $F_z$ and that $U_{n,s,N,\omega}$ is a generalized incomplete U-statistic with kernel $h=h(Z_1,\dots,Z_{s}; \omega)$. Let $\theta=\E[h]$, $\zeta_s=\V(h)$, and $\zeta_{1,\omega} = \E[g^2(Z_1)]$, where $g(z)= \E[h(z, Z_2,\dots,Z_{s}; \omega)]-\theta$ and let $p={N}/{{n\choose s}}$. Suppose further that $\zeta_s<\infty$ and $\zeta_{1,\omega}>0$. Then for any $ 0 < \eta_0 < 1/2$, 
		\[
		\begin{aligned}
			& \sup_{z\in \R}\left|\P\left\{\frac{U_{n,s,N,\omega}-\theta}{\sqrt{s^2\zeta_{1,\omega}/n}}\leq z\right\}- \Phi(z)\right| \\
			& \leq C\left\{ \frac{\E|g|^3}{n^{1/2}\zeta_{1,\omega}^{3/2}} + \left\{\frac{s}{n} \left(\frac{\zeta_s}{s\zeta_{1,\omega}}-1\right)\right\}^{1/2}  + \left\{ \frac{n}{N}\left(1-p\right) \frac{\zeta_s}{s\zeta_{1,\omega}}\right\}^{1/2} + N^{-1/2 + \eta_0}\right\}.
		\end{aligned}\]
		
	\end{theorem}
	
	The proof of \cref{be-2} is provided in \cref{app:C-BR}. 
	The preceding theorems indicate that for both infinite-order and generalized U-statistics, when incomplete versions of these estimators are used, these statistics remain asymptotically as efficient as the complete forms so long as $n=o(N)$. 
	Comparing Theorems \ref{be-1} and \ref{be-2}, note that these bounds differ only by the inclusion of two additional final terms in the sum, where $N^{-1/2+\eta_0}$ is due to the variation of $\hat{N}$ with respect to $N$ and $\left\{\frac{n}{N}(1-p)\frac{\zeta_s}{s\zeta_{1,\omega}}\right\}^{1/2}$ is due to incompleteness, which is close to 0 in such large-$N$ settings.
	
	However, in small-$N$ settings, this term can become quite large, leading to a bound nearing or even exceeding 1, thereby making it of little use. \cref{be-3} below provides improved Berry-Esseen bounds for such small-$N$ settings where relatively few base learners are employed. To achieve this, rather than writing the estimators as linear statistics plus a small additional manageable term, we take an alternative approach that views incomplete U-statistics as complete U-statistics plus a remainder. This strategy is similar to that used in \cite{chen2019randomized} who recently derived non-asymptotic Gaussian approximation error bounds for high-dimensional, incomplete U-statistics, but for kernels with fixed (finite) rank.  Proofs of the following results are provided in Appendix \ref{app:C-BR}.

	\begin{theorem}
		\label{be-3}
		
		Suppose that $Z_1,\dots, Z_n$ are i.i.d.\ from $F_Z$ and that $U_{n,s,\omega, N}$ is a generalized incomplete U-statistic with kernel $h=h(Z_1,\dots, Z_s;\omega)$. Let $\theta=\E[h]$, $\zeta_s =\V(h)$, and $\zeta_{1,\omega} = \E[g^2(Z_1)]$, where $g(z) = \E[h(z,Z_2,\dots, Z_s; \omega)]-\theta$. Suppose further that $\zeta_s<\infty$ and $\zeta_{1,\omega}>0$.  If $\E[|h-\theta|^{2k}]/\E^2[|h-\theta|^k]$ is uniformly bounded for $k=2,3$ and for all $s$. Then for any $0<\eta_0<1/2$, 
		\begin{equation*}
			\label{eq:Berry-IU}
			\begin{aligned}
				& \sup_{z\in \R}\left|\P\left\{\frac{U_{n,s,N,\omega}-\theta}{\sqrt{{s^2\zeta_{1,\omega}}/{n}+ {\zeta_s}/{N}}}\leq z\right\} -\Phi(z)\right|\\
				& \leq C\left\{\frac{\E|g|^3}{n^{1/2}(\E|g|^2)^{3/2}}  + \frac{\E|h-\theta|^3}{N^{1/2}(\E|h-\theta|^2)^{3/2}} + 
				\left\{\frac{s}{n} \left(\frac{\zeta_s}{s\zeta_{1,\omega}}-1\right)\right\}^{1/2} \right.\\
				& ~~~~\left.+\, N^{-1/2+\eta_0} +  \left(\frac{s}{n}\right)^{1/3}  \right\}
			\end{aligned}
		\end{equation*}	
		for some constant $C>0$.
		
	\end{theorem}
	
	Here we see that when $s$ is fixed, the Berry-Esseen bound is on the order of $n^{-1/3}$. When $s$ grows with $n$, the bound converges to zero as long as $s/n \to 0$ and $N \to \infty$ with some mild conditions on $h$. 	
	The fundamental task in producing this result is to show that the convolution of the two independent sequences approaches a normal distribution.  A number of approximations are required, though we give a nearly sharp bound on each in order to provide the Berry-Esseen bound shown.  As noted earlier, \cite{chen2019randomized} recently investigated a similar setup for higher-dimensional kernels assumed to be of fixed rank.  In our case, the use of an infinite-order kernel injected with extra randomness introduces additional technical difficulties, though the restriction to one-dimensional settings allows us to incorporate more useful concentration inequalities.  Thus, even for kernels assumed to have a fixed, finite rank, the result above is sharper than that provided in \cite{chen2019randomized} for the one-dimensional setting.  
	
	As an additional benefit, we note that the bound above consisting of a five-term sum contains insightful terms not produced in \cite{chen2019randomized}.  In particular, the second term corresponds to the bound that would be available for an estimator that takes an average of i.i.d.\ random variables, while the first term plus the third term gives the bound for the complete infinite-order U-statistic setting.  This leads to the very natural intuition that when $N$ is quite small, the bound produced is approximately what would be expected by averaging over independent base learners whereas when $N$ is large, the bound is approximately what we would expect for a complete infinite-order U-statistic.  Again, $N^{-1/2+\eta_0}$ is due to the variation of $\hat{N}$ with respect to $N$.  To see where the final term $(s/n)^{1/3}$ comes from, we now delve into the proof details in the following subsection.

	
	\subsection{A Tighter Bound }
	
	In order to obtain the previous bounds, we first condition on $Z_1,\dots, Z_n$ and obtain a Berry-Esseen bound for the difference between the infinite-order forms of incomplete and complete U-statistics, $U_{n,s,N} - U_{n,s}$. 
	The terms involved in this bound are themselves infinite-order U-statistics with kernels that are power functions of the original kernel $h$.  We make use of Chebyshev's inequality to replace those infinite-order U-statistics by their population mean, the application of which requires no particular assumptions on the tail behavior of the kernel.  This approach, however, leads to the non-optimal term of $(\frac{s}{n})^{1/3}$.  We thus conclude our discussion on Berry-Esseen Theorems by showing in this final subsection that placing additional assumptions on the kernel $h$ can allow the application of sharper concentration inequalities that can therefore allow the term $(\frac{s}{n})^{1/3}$ to be replaced by $(\frac{s}{n})^{1/2}$. 
	
	\begin{theorem}
		\label{be-4}
		Suppose that $Z_1,\dots, Z_n$ are i.i.d.\ from $F_Z$ and that  $U_{n,s,N,\omega}$ is a generalized incomplete U-statistic with kernel $h=h(Z_1,\dots, Z_s;\omega)$. Let $\theta=\E[h]$, $\zeta_s = \V(h)$ and $\zeta_{1,\omega} = \E[g^2(Z_1)]$, where $g(z) = \E[h(z,Z_2,\dots, Z_s; \omega)]-\theta$. Suppose further that $\zeta_s<\infty$ and $\zeta_{1,\omega}>0$. 
		If $|h-\theta|^k$ is sub-Gaussian after standardization with variance proxy that is uniformly bounded for $k=2,3$ and all $s$, then for any $0<\eta_0,\eta<1/2$,
		\begin{equation*}
			\label{eq:Berry-IU-subgaussian}
			\begin{aligned}
				& \sup_{z\in \R}\left|\P\left\{\frac{U_{n,s,N,\omega}-\theta}{\sqrt{{s^2}\zeta_{1,\omega}/n+ {\zeta_s}/{N}}}\leq z\right\} -\Phi(z)\right|\\
				& \leq C\left\{\frac{\E|g|^3}{n^{1/2}(\E|g|^2)^{3/2}}  + \frac{\E|h-\theta|^3}{N^{1/2}(\E|h-\theta|^2)^{3/2}} + \left[\frac{s}{n} \left(\frac{\zeta_s}{s\zeta_{1,\omega}}-1\right)\right]^{1/2} \right.\\ 
				&~~~~  \left. +\, N^{-1/2+\eta_0} + \left(\frac{s}{n}\right)^{-1/2+\eta} \right\},
			\end{aligned}
		\end{equation*}
		where $C>0$ is some constant.
		
	\end{theorem}
	Note that since there is a trade-off between the probability and concentration bound, larger $\eta$ or $\eta_0$ requires a larger $n$ to ensure the above inequality holds. Proof details of \cref{be-4} are given in \cref{app:C-TB} for the incomplete infinite-order U-statistic setting; the extension to \emph{generalized} incomplete U-statistics follows in an identical fashion.

	\section{Discussion}\label{discussion}
	
	The previous sections establish distributional results for random forest estimators, which take the form of \emph{generalized} U-statistics.  We showed that under mild regularity conditions, such estimators tend to a normal distribution so long as $\frac{s}{n}\left( \frac{\zeta_s}{s\zeta_{1,\omega}} -1 \right)\to 0$.  When kernels are well-behaved, this thus implies that subsamples may be taken on the order of $n$ while retaining the asymptotic normality of the estimator. In practice, we expect that this condition is often most naturally satisfied by subsampling at a slower rate with $s=o(n)$ and ensuring that the corresponding variance ratio ${\zeta_s}/{s\zeta_{1,\omega}}$ is bounded. In Section 3 we showed that the variance ratio is well-behaved for a number of nearest-neighbor-type base learners.  In general though, such behavior is not well-understood, particularly for adaptive learners and such investigations present promising opportunities for future research. In \cref{berry-esseen} we provide Berry-Esseen bounds to quantify the proximity of these estimators to the normal distribution.  \cref{be-1} provides the sharpest bound to date on this rate for complete, infinite-order U-statistics, while the bounds that follow are each the first of their kind.
	
	Throughout this work, asymptotic results for incomplete U-statistics are often based on the asymptotic result for the complete analogue.  Indeed, this is how the condition on the variance ratio arises. While natural, this approach is not strictly necessary for establishing asymptotic normality of the incomplete version, particularly when the number of subsamples is quite small. Further discussion on this somewhat special case is provided in the appendix.


	\section*{Acknowledgements}
	We would like to thank Larry Wasserman for helpful conversations and feedback.  
	
	\bibliographystyle{unsrt} 
   
   \bibliographystyle{Chicago}

	\appendix
	\section{Proofs in  \cref{background}}
	\label{app:background}
	
	Here we provide a fuller discussion of the previously established central limit theorems for randomized, incomplete, infinite-order U-statistics, paying particular attention to the relationship between the projection method utilized and the resulting subsampling rate necessary in order to retain asymptotic normality.  As noted in \cref{background}, \cite{Mentch2016} provided one such theorem, but with somewhat strict conditions.  First, the authors require that for all $\delta>0$,
	\[
	\frac{1}{\zeta
		_{1,\omega}} \int_{\left|h_1(z)\right|\geq \delta\sqrt{n\zeta_{1,\omega}} }  h_1^2(z) \,dP \to 0 \quad  (n\to \infty)
	\]
	where $h_{1}(z) =  \E[h(z,Z_2,\dots,Z_{s};\omega)]-\theta$.  Note however that so long as $\E[h^2(Z_1, \dots , Z_s;\omega)]<\infty$, 
	\[
	\frac{1}{\zeta_{1,\omega}} {\int}_{\left| h_{1}\right| \geq \delta\sqrt{n\zeta_{1,\omega}}}  h^2_{1}(Z)\,dP 
	= {\int}_{ \left|\frac{h_{1}}{\sqrt{\zeta_{1,\omega}}}\right|\geq \delta\sqrt{n} } \left(\frac{h_{1}}{\sqrt{\zeta_{1,\omega}}}\right)^2\,dP
	\]
	automatically tends to 0 as $n \to \infty$ and thus this condition is redundant for kernels assumed to have finite second moment.

	In Section 2, we noted that there is strong reason to suspect that a subsampling rate of $s=o(n^{1/2})$ is the largest possible when the results are established via H\'ajek projections.  We now elaborate on that point here.  
	
	Let $\mathcal{S}$ denote the set of all variables of the form  $\sum_{i=1}^n g_i(Z_i)$ for arbitrary measurable functions $g_i:\R^d \mapsto \R$ with $\E[ g_i^2(Z_i)]<\infty$ ($i = {1,\dots, n}$). The H\'ajek projection of $U_{n,s}$ onto $\mathcal{S}$ is
	\[  
	\hat{U}_{n,s} =  \theta+ \frac{s}{n}\sum_{i=1}^n h_{1}(Z_i).
	\]
	Now, by the central limit theorem for i.i.d case,  we have 
	$
	{\sqrt{n}\hat{U}_{n,s}}/{\sqrt{s^2\zeta_{1}}}\rightsquigarrow N(0,1)
	$
	and thus by Theorem 11.2 in \cite{vaart_1998}, to obtain the asymptotic normality of U-Statistic, it is sufficient to demonstrate that ${\V(U_{n,s})}/{\V(\hat{U}_{n,s})}\to1$.  This is straightforward when the rank of the kernel is fixed but requires more careful attention whenever $s$ is allowed to grow with $n$. The variance of the U-statistic is
	\begin{equation*}
		\begin{aligned}
			\V(U_{n,s}) 
			&= {n\choose s}^{-1} \sum_{\beta}\sum_{\beta'} \Cov( h(Z_{\beta_1},\dots, Z_{\beta_{s}}), h(Z_{\beta_1'},\dots,   Z_{\beta'_{s}})) \\
			& = {n \choose s}^{-1} \sum_{j=1}^{s} {s \choose j}{n-s \choose s-j} \zeta_j \\
			& = \sum_{j=1}^{s} \frac{{s!}^2}{j!(s-j)!^2}\frac{(n-s)\cdots (n-2s+j+1)}{n(n-1)\cdots(n-s+1)}\zeta_j
		\end{aligned}
	\end{equation*}
	where $\beta$ indexes subsamples of size $s$, and the variance of $\hat{U}_{n,s}$ is
	\begin{equation*}
		\V(\hat{U}_{n,s})= \frac{s^2}{n}\V\left( h_1(Z_1)\right) =  \frac{s^2}{n}\zeta_1.
	\end{equation*}
	The variance ratio is then $ {\V(U_{n,s})}/{\V(\hat{U}_{n,s})} = {(a_n+b_n)}/{c_n}$, where 
	\begin{align*}
		a_n &=\frac{s^2}{n}\frac{(n-s)\cdots (n-2s+2) }{(n-1)\cdots (n-s+1)}\zeta_1, \\
		b_n &={n \choose s}^{-1}\sum \limits_{j=2}^s {s \choose j}{n-s\choose s -j}\zeta_j, \\
		c_n &=\frac{s^2}{n}\zeta_1.
	\end{align*}
	Thus, in order for the variance ratio to converge to 1, it suffices to show ${a_n}/{c_n}\to 1$ and ${b_n}/{c_n}\to 0$. To transform these two conditions with respect to $s$ and $n$, we introduce the following lemmas.  
	
	\begin{lemma}[\citep{lee1990u}]
		\label{lemma1}
		For $1 \leq c\leq d\leq s$, ${\zeta_s}/{c} \leq {\zeta_d}/{d}$.
	\end{lemma}
	
	\begin{lemma}\label{lemma2}
		Let ${H}(n,s)=\left[\frac{(n-s)\cdots (n-2s+2)}{(n-1)\cdots (n-s+1)}\right]$, then $s/\sqrt{n}\to 0$ if and only if $H(n,s) \to 1$.
	\end{lemma}
	
	\begin{proof}
		When $s/\sqrt{n}\to 0$, we have 
		\[
		\begin{aligned}
			H(n,s)
			& \geq \left[\frac{n-2s+2}{n-1}\right]^{s-1} \\
			& = \exp\left[ (s-1)\log \left(1- \frac{2s-3}{n-1}\right)\right] \\
			& \approx \exp\left[-\frac{2s^2}{n}\right] \rightarrow 1.
		\end{aligned}
		\]
		If there exists a subsequence $\{s'\}$ such that $s' /\sqrt{n'}\geq c$ for some constant $c>0$, then 
		\[
		\begin{aligned}
			H(n',s')
			& \leq \left[\frac{n'-{3s'}/{2}+1}{n-{s'}/{2}}\right]^{s'-1} \\
			& = \exp \left[ (s'-1)\log\left( 1- \frac{s'-1}{n-{s'}/{2}}\right)\right] \\
			& \approx \exp\left[-\frac{s'^2}{n'}\right] <1.
		\end{aligned}
		\]
	\end{proof}
	
	Now, we can transform the conditions on $a_n,b_n$ and $c_n$ into conditions on $n$ and $s$. Note that 
	\begin{equation*}
		{a_n}/{c_n} = H(n,s)
	\end{equation*}
	and 
	\begin{equation*}
		\begin{aligned}
			{b_n}/{c_n} 
			&={ n-1\choose s-1}^{-1} \left\{ \sum \limits_{j=1}^{s-1}\frac{1}{j+1} {s-1\choose j}{(n-1)-(s-1)\choose (s-1)-j}\frac{\zeta_{j+1}}{\zeta_1} \right\} \\
			& = { n-1\choose s-1}^{-1}  \left\{\sum \limits_{j=1}^{s-1} {s-1\choose j}{(n-1)-(s-1)\choose (s-1)-j}\frac{\zeta_{j+1}}{(j+1)\zeta_1} \right\}\\
			&  \geq 1 - \left[\frac{(n-s)\cdots (n-2s+2)}{(n-1)\cdots (n-s+1)}\right] \\
			& = 1 - H(n,s).
		\end{aligned}
	\end{equation*}
	Due to Lemma \ref{lemma2}, $s/\sqrt{n}\to 0$ is the necessary condition for ${b_n}/{c_n}\to 0$ and ${a_n}/{c_n}\to 1$. Thus, if we utilize the H\'ajek projection and follow the above approach in establishing that the variance ratio converges to 1, there is no apparent way to relax the condition that $s/\sqrt{n}\to 0$.  On the other hand, the H-decomposition we use in \cref{normality} provides a finer approach and a better method for comparing the variance of $U_{n,s}$ and $\hat{U}_{n,s}$ thereby allowing for a faster subsampling rate to be employed.
	
	\vspace*{-10pt}

	\section{Proofs in Section \ref{normality}}
	\label{app:normality}
	
	\subsection{H-decomposition}
	\label{app:B-HD}
	
	Distributional results for U-statistics are typically established via projection methods whereby some projection $\hat{U}$ is shown to be asymptotically normal with $|U-\hat{U}| \to 0$ in probability. The most popular projections are the H\'ajek projection and the H-decomposition. We show in \cref{app:background} that the approach of  H\'ajek projection always requires $s/\sqrt{n}\to 0$ undesirably.
	Alternatively, the H-decomposition provides a representation of U-statistics in terms of sums of other uncorrelated U-statistics of rank $1,\dots, s$.  The form of this decomposition presented here is derived by \cite{Hoeffding1961}.
	We illustrate those techniques in the setting of the original U-statistic $U_{n,s}$ for simplicity and then extend them to the generalized complete U-statistic $U_{n,s,\omega}$.
	Let
	\begin{equation*}
		h_c(z_1,\dots, z_c) = \E[h(z_1,\dots, z_c,Z_{c+1},\dots,Z_{s})]-\theta,
	\end{equation*}
	and define kernels $h^{(1)},h^{(2)},\dots, h^{(s)}$ of degree $1,\dots,s$ recursively as 
	\begin{align}
		h^{(1)} &= h_1(z_1) \nonumber \\
		h^{(2)}  &= h_2(z_1,z_2)- h_1(z_1)-h_1(z_2) \nonumber \\
		&    \vdots \label{eq:hd} \\
		h^{(s)} &= h_s(z_1,\dots, z_s)-\sum \limits_{j=1}^{s-1}\sum\limits_{(s,j)}h^{(j)}(z_{i1},\dots, z_{ij}). \nonumber
	\end{align}
	These kernel functions have many important and desirable properties, a sample of which are enumerated in the following proposition.
	\begin{proposition}[\citep{lee1990u}]
		\label{prop1}
		
		For $h^{(j)}$, $j =1,\ldots,s$ defined as above, we have 
		\begin{description}
			\item[1.] For $c=1,\ldots, j-1$, 
			$\E[h^{(j)}(z_1,\dots ,z_c,Z_{c+1},\dots, Z_j)]= 0.$
			
			\item[2.] $\E[h^{(j)}(Z_1,\ldots,Z_j)] = 0.$
			
			\item[3.] Let $j<j'$ and $S_1$ and $S_2$ be a $j$-subset of $\left\{Z_1,\dots, Z_n\right\}$ and a $j'$-subset of $\left\{Z_1,\dots,Z_n\right\}$ respectively, then
			$ \Cov (h^{(j)}(S_1), h^{(j')}(S_2)) = 0$.
			
			\item[4.] Let $S_1\neq S_2$ be two distinct $j$-subsets of $\left\{Z_1,\dots, Z_n\right\}$, then
			$ \Cov(h^{(j)}(S_1), h^{(j')}(S_2)) = 0$.
		\end{description}
	\end{proposition}
	$h=h(Z_1,\dots, Z_n)$ can be written as $h = \sum_{j=1}^s\sum_{(s,j)}h(Z_{i1},\dots, Z_{ij})$
	and the expression of $U_{n,s}$ now follows easily as
	\begin{equation*}
		\label{eq:ru}
		\begin{aligned}
			U_{n,s}-\theta 
			& = {n \choose s}^{-1}\sum_{(n,s)}h_s(Z_{i1},\dots, Z_{is})\\
			& =   {n \choose s}^{-1}\sum_{(n,s)} \left\{\sum_{j=1}^{s}\sum_{(s,j)}h^{(j)}(Z_{i1},\dots, Z_{ij})\right\} \\
			&= \sum_{j=1}^s {n \choose s}^{-1}\sum_{(n,s)}\sum_{(s,j)}h^{(j)}(Z_{i1},\dots, Z_{ij}) \\
			& =  \sum_{j=1}^s {s \choose j}H_n^{(j)}
		\end{aligned}
	\end{equation*}
	where $H^{(j)}_n = {n \choose j}^{-1}\sum_{(n,j)}h^{(j)}(Z_{i1},\dots ,Z_{ij})$
	is itself a U-statistic, the usefulness of which lies in the fact that $H_n^{(j)}~(j = 1,\dots, n)$ are uncorrelated and the terms in $H_n^{(j)}$ are also uncorrelated.  Because of the properties above, the variance of the kernel is 
	\begin{equation}
		\label{h-variance}
		\V(h) = \V \left\{\sum_{j=1}^{s}\sum_{(s,j)}h^{(j)}(Z_{i1},\dots, Z_{ij})\right\} = \sum_{j=1}^s {s \choose j} V_j 
	\end{equation}
	where $V_j = \V(h^{(j)}(Z_{i1},\dots,Z_{ij}))$. Similarly, the variance of the U-statistic can be written as
	\begin{equation}
		\label{u-variance}
		\V(U_{n,s}) = \V\left\{\sum_{j=1}^s {s \choose j}H_n^{(j)}\right\}=  \sum \limits_{j=1}^s {s \choose j}^2 {n\choose j}^{-1} V_j.
	\end{equation}
	Note that the first-order term $sH_n^{(1)}$ is exactly the same as in the H\'ajek projection $\hat{U}_{n,s}$, but the H-decomposition provides a convenient alternative representation of U-statistics as well as their variance.  In \cref{normality}, we exploit this fact to derive a tighter and more general central limit theorem for generalized U-statistics.

	
	\subsection{Proofs of Asymptotic Normality}
	\label{app:B-AN}

	\noindent \textbf{Proof of \cref{AN-u}:  } The generalized complete U-statistic and the base learner can be written in terms of the new kernel functions $h^{(1)},\dots, h^{(s)}$ defined in relation to the H-decomposition.  Let $V_{i,\omega} = \V(h^{(i)})$ for $i=1,\ldots, s-1$, $V_s = \V(h^{(s)})$ and define
	\[ 
	V_{s,\omega} = \V\left \{h_s(Z_1,\dots, Z_s)-\sum \limits_{j=1}^{s-1}\sum\limits_{(s,j)}h^{(j)}(Z_{i1},\dots, Z_{ij})\right \} .
	\]
	These new kernels $h^{(1)},\dots, h^{(s)}$ still retain the desirable properties in Proposition \ref{prop1}. Thus, similar to \cref{h-variance}, \cref{u-variance}, we have the following expressions for the variance of the kernel and generalized U-statistic:
	\begin{align}
		\label{Variances}
		&\V(h)  = \zeta_{s} = \sum_{j=1}^{s-1} {s\choose j} V_{j,\omega} + V_s, \nonumber\\
		&\V(\hat{U}_{n,s,\omega}) = \frac{s^2}{n}V_{1,\omega} = \frac{s^2}{n}\zeta_{1,\omega},  \\
		&\V(U_{n,s,\omega}) = \sum_{j=1}^{s-1} {s\choose j}^2 {n\choose j}^{-1} V_{j,\omega}+ {n \choose s}^{-1} V_s . \nonumber
	\end{align}
	The sequence $ \hat{U}_{n,s,\omega}/\sqrt{s^2\zeta_{1,\omega}/n}$ converges weakly to $N(0,1)$ by the central limit theorem since  $\hat{U}_{n,s,\omega} = \frac{s}{n}\sum_{i=1}^n h_1(Z_i)$ is a sum of i.i.d.\ random variables for each $s$, which satisfies Lindeberg's condition automatically. From \cref{Variances}, we have 
	\begin{equation}
		\label{ieq}
		\begin{aligned}
			\frac{\V(U_{n,s,\omega})}{\V(\hat{U}_{n,s,\omega})} 
			& = \left(\frac{s^2}{n}V_{1,\omega}\right)^{-1} \left \{\sum_{j=1}^{s-1} {s\choose j}^2 {n\choose j}^{-1} V_{j,\omega}+ {n \choose s}^{-1} V_s\right \}\\
			& \leq 1 + \left(\frac{s^2}{n}V_{1,\omega}\right)^{-1}\frac{s^2}{n^2}\left\{\sum_{j=2}^{s-1} {s\choose j}V_{j,\omega} + V_s\right\} \\
			&  \leq 1 + \frac{s}{n}\frac{\zeta_s}{s\zeta_{1,\omega}} \to 1.
		\end{aligned}
	\end{equation}  
	Thus by Theorem 11.2 in \cite{vaart_1998}, we obtain 
	$\frac{U_{n,s,\omega}-\theta}{\sqrt{s^2\zeta_{1,\omega}/n}} \rightsquigarrow N(0,1)$. \hfill $\blacksquare$

	
	\noindent \textbf{Proof of \cref{AN-iu}: } Without loss of generality, let $\theta=0$ and observe that	\[ 
	\begin{aligned}
		U_{n,s,N,\omega} 
		& = \frac{1}{\hat{N}}\sum_{(n,s)}\rho h(Z_{i1},\ldots, Z_{is}; \omega)\\
		& = \frac{N}{\hat{N}} \left[U_{n,s,\omega} +  \frac{1}{N}\sum_{(n,s)}(\rho-p)h(Z_{i1},\ldots, Z_{is}; \omega)\right] \\
		& = \frac{N}{\hat{N}}\left[A_n + B_n\right],
	\end{aligned} \]
	where $A_n$ and $B_n$ are uncorrelated and $\V(B_n)= (1-p)\zeta_s/N:=d^2_{n,S,N}$. Since ${N}/{\hat{N}}\xrightarrow[]{p}1$ as $N\to \infty$, we need only show the distributional result of $A_n + B_n$ and apply Slutsky's theorem to obtain that of $U_{n,s,N,\omega}$.

	First, consider the case that $p = N/{n\choose s} \not \to 0$. Since $\smash{\frac{s}{n}\frac{\zeta_s}{s\zeta_{1,\omega}} \to 0}$, by \cref{AN-u} we have $A_n/\sqrt{s^2\zeta_{1,\omega}/n}\rightsquigarrow  N(0,1)$. Moreover, we have 
	\begin{equation*}
		\frac{\V(B_n)}{\V(A_n)} \to \frac{ N^{-1}(1-p)\zeta_s}{s^2\zeta_{1,\omega}/n} \leq \frac{n^2}{Ns^2}\cdot \frac{s}{n}\frac{\zeta_s}{s\zeta_{1,\omega}} \to 0.
	\end{equation*}
	Thus $U_{n,s,N,\omega}/\sqrt{s^2\zeta_{1,\omega}/n}\to N(0,1)$, implying \cref{eq:M2}.
	
	Now, consider the case that $p\to 0 $ and define
	\begin{equation*}
		\begin{aligned}
			\phi_{A_n+B_n}(t) =
			& ~\E\left[\exp\left(it \left( \frac{s^2\zeta_{1,\omega}}{n} +  \frac{\zeta_s}{N}\right)^{-1/2}(A_n+B_n)\right)\right] \\
			& = \E \left[    \exp\left(it\left(\frac{s^2\zeta_{1,\omega}}{n}+\frac{\zeta_s}{N}\right)^{-1/2}A_n\right) \right.\times \\ 
			& ~~~~~~~~~ \left.\E\left[\exp\left(it\left(\frac{s^2\zeta_{1,\omega}}{n}+\frac{\zeta_s}{N}\right)^{-1/2}B_n\right)\mid Z_1,\ldots, Z_n; \omega\right]\right]  \\
			& = \E\left[\hat{\phi}_{A_n}(t)\hat{\phi}_{B_n}(t)\right].
		\end{aligned}
	\end{equation*}
	We will show $\phi_{A_n+B_n}\to e^{-\frac{t^2}{2}}$, of which the key step is to demonstrate that  $\hat{\phi}_{B_n}(t)$ is well behaved with high probability.
	Note that $U_2 = {n\choose s}^{-1}\sum_{(n,s)}h^2(Z_{i1},\dots, Z_{is};\omega)$ is a complete U-statistic with kernel $h^2$.  For any $\epsilon>0$, by Chebyshev's inequality, we have 
	\[
	\P\left\{|U_2-\E[h^2]|\geq \epsilon  \E[h^2]\right\} \leq \frac{s}{n}\frac{\E[h^4]}{\epsilon^2\E^2[h^2]}\leq \frac{s}{n}\frac{C}{\epsilon^2},
	\]
	which indicates that  $U_2/\zeta_s(=U_2/\E[|h|^2])\xrightarrow[]{p}1$.  $U_3/\E[|h|^3]\xrightarrow[]{p}1 $ also holds  by a similar argument, where  $U_3 = {n\choose s}^{-1}\sum_{(n,s)} |h(Z_{i1},\dots, Z_{is};\omega)|^3$. Let  $D=\left\{U_2/\E[h^2] \in [1-\delta, 1+\delta]\right\}\cap \left\{U_3/\E[|h|^3] \in [1-\delta, 1+\delta]\right\}$.  Then for any $\delta, \epsilon>0$,  $D$  holds with probability at least $1-\epsilon$ for $n$ sufficiently large.  Let $\hat{d}_{n,s,N} = \left[(1-p)U_2/N\right]^{1/2}$ and consider $B_n/\hat{d}_{n,s,N}\mid Z_1,Z_2,\dots, Z_n;\omega$, which is a sum of independent random variables. Thus, to establish asymptotic normality, it suffices to check the Lyapounov's condition.  We have
	\begin{equation*}
		\mathcal{L} = \frac{{n\choose s}^{-1}\sum_{(n,s,)} |h|^3}{\left({n\choose s}^{-1}\sum_{(n,s)}|h|^2\right)^{3/2}}  \frac{1-2p+2p^2}{\sqrt{N(1-p)}} =  \frac{1-2p+2p^2}{\sqrt{N(1-p)}} \frac{U_3}{U_2^{3/2}}.
	\end{equation*}
	Thus, as $N\to \infty$ and $p\to 0$, 
	\begin{equation*}
		\mathcal{L} \leq \frac{1+\delta}{(1-\delta)^{3/2}} \frac{\E[|h|^3]}{\E^{3/2}[|h|^2]} \frac{1-2p+2p^2}{\sqrt{N(1-p)}} \to 0
	\end{equation*}
	uniformly with respect to $Z_1,\dots, Z_n$ and $\omega$ over $D$ and hence we have
	\[
	\E \left[\exp\left(iuB_n/(\hat{d}_{n,s,N})\right)\mid Z_1,\dots,Z_n;\omega\right] \to e^{-\frac{u^2}{2}}
	\]
	uniformly over any finite interval  of $u$ and uniformly with respect to $Z_1,\dots, Z_n$ and $\omega$ over $D$.  Letting the interval be $[0, t\sqrt{(1+\delta)}]$,  we have 
	\begin{equation*}
		\begin{aligned}
			&~~~~ \left| \hat{\phi}_{B_n}(t) - \exp\left(\frac{t^2}{2}\left(\left(  1 -p\right)\cdot \frac{\zeta_s/N}{ s^2\zeta_{1,\omega}/{n} + \zeta_s/N}\right)\zeta_s^{-1}U_2\right)\right| \\
			& = 	\left|\E\left[ \exp\left(t\left(\left(  1 -p\right)\cdot \frac{\zeta_s/N}{ s^2\zeta_{1,\omega}/{n} + \zeta_s/N}\right)^{1/2}(\zeta_s^{-1}U_2)^{1/2} \cdot B_n/\hat{d}_{n,s,N}  \right)\mid Z_1,\dots, Z_n;\omega\right] \right. \\
			& \left. - \exp\left(-\frac{t^2}{2}\left(\left(  1 -p\right)\cdot \frac{\zeta_s/N}{ s^2\zeta_{1,\omega}/{n} + \zeta_s/N}\right)\zeta_s^{-1}U_2\right)\right|\leq \epsilon.
		\end{aligned}
	\end{equation*}
	over $D$ for $n$, $N$ sufficiently large. Now, let 
	\begin{equation*}
		\phi_B(t) =  \exp\left(-\frac{t^2}{2}\left(\left(  1 -p\right)\cdot \frac{\zeta_s/N}{ s^2\zeta_{1,\omega}/{n} + \zeta_s/N}\right)\right).
	\end{equation*}
	Then by the uniform continuity of $e^x$ over any finite interval,  there exists $\delta'= O(\delta)$ such that 
	\begin{equation*}
		\begin{aligned}
			&~~~~	\left|\exp\left(-\frac{t^2}{2}\left(\left(  1 -p\right)\cdot \frac{\zeta_s/N}{ s^2\zeta_{1,\omega}/{n} + \zeta_s/N}\right)\zeta_s^{-1}U_2\right) - \phi_B(t)\right|  \leq \delta'
		\end{aligned}
	\end{equation*}
	over $D$. Finally, for $n$ and $N$ sufficiently large, we have
	\begin{equation}
		\label{app1}
		\left|\hat{\phi}_{B_n}(t) - \phi_B(t)\right|1_{D} \leq  (\epsilon+\delta')1_{D}.
	\end{equation}

	\noindent Next, consider $\hat{\phi}_{A_n}$. Since $A_n/\sqrt{s^2\zeta_{1,\omega}/n} \rightsquigarrow  N(0,1)$ by \cref{AN-u}, then  
	\[\E\left[\exp\left(iuA_n/\sqrt{s^2\zeta_{1,\omega}/n}\right)\right] \to e^{-\frac{u^2}{2}}\]
	uniformly over any finite interval of $u$. Let the interval be $[0,t]$, then for $n$ sufficiently large,  we have 
	\begin{equation*}
		\left|\E\left[\exp\left(it\left(\frac{s^2\zeta_{1,\omega}/n}{s^2\zeta_{1,\omega}/n + \zeta_s/N}\right)^{1/2}\cdot A_n/\sqrt{s^2\zeta_{1,\omega}/n}\right)\right] -  e^{-\frac{t^2}{2}\frac{s^2\zeta_{1,\omega}/n}{s^2\zeta_{1,\omega}/n + \zeta_s/N} } \right| \leq \epsilon.
	\end{equation*}
	Let $\phi_A(t) = e^{-\frac{t^2}{2}\frac{s^2\zeta_{1,\omega}/n}{s^2\zeta_{1,\omega}/n + \zeta_s/N} }$ and 
	consequently, we have 
	\begin{equation}
		\label{app2}
		\left|\E \left[\hat{\phi}_{A_n}(t)\right] - \phi_A(t)\right| \leq \epsilon.
	\end{equation}
	Combining \cref{app1} and \cref{app2} gives
	\begin{align*}
		\left|\phi_{A_n+B_n}(t) -\phi_A(t)\phi_B(t)\right|  
		& = \left|\E\left[\hat{\phi}_{A}(t)\hat{\phi}_{B_n}(t)\right]-\phi_A(t)\phi_B(t)\right| \\
		& \leq  \left|\E\left[\hat{\phi}_{A_n}(t)\phi_{B}(t)1_D\right]-\phi_A(t)\phi_B(t)\right| \\
		& ~~~~~~~~  \left|\E\left[\hat{\phi}_{A_n}(t)\left(\hat{\phi}_{B_n}(t)-\phi_B(t)\right)1_D\right]\right|+ \epsilon \\
		& \leq  \left|\E\left[\hat{\phi}_{A_n}(t)\phi_{B}(t)1_D\right]-\phi_A(t)\phi_B(t)\right|  + (\epsilon+\delta') + \epsilon\\
		&   \leq \left|\E\left[\hat{\phi}_{A_n}(t)1_D\right]-\phi_A(t)\right| + 2\epsilon + \delta' \stepcounter{equation}\tag{\theequation}\label{arg1}  \\
		& \leq \left|\E\left[\hat{\phi}_{A_n}(t)\right]-\phi_A(t)\right|+ \left|\E\left[\hat{\phi}_{A_n}(t)1_{D^c}\right]\right|+  2\epsilon + \delta' \\
		& \leq \epsilon + \epsilon + 2\epsilon + \delta' \\
		& = 4\epsilon+\delta'.
	\end{align*}
	
	\noindent Moreover, we have 
	\begin{equation*}
		\begin{aligned}
			\phi_A(t)\phi_B(t) 
			& = \exp\left(-\frac{t^2}{2}\left[\frac{s^2\zeta_{1,\omega}/n}{s^2\zeta_{1,\omega}/n + \zeta_s/N}  +  \left(\left(  1 -\frac{N}{{n\choose s}}\right)\cdot \frac{\zeta_s/N}{ s^2\zeta_{1,\omega}/{n} + \zeta_s/N}\right)\right]\right)
			\\ 
			& = e^{-\frac{t^2}{2}} \cdot \exp\left(-\frac{t^2}{2}\left[ \frac{N}{{n\choose s}} \frac{\zeta_s/N}{ s^2\zeta_{1,\omega}/{n} + \zeta_s/N}\right]\right) \\
			& \to e^{-\frac{t^2}{2}}
		\end{aligned}
	\end{equation*}
	which implies that for $n$, $N$ sufficiently large, we have 
	\begin{equation}
		\label{arg2}
		\left|\phi_A(t)\phi_B(t) -e^{-\frac{t^2}{2}}\right| \leq \epsilon.
	\end{equation}
	Combining \cref{arg1} and \cref{arg2} yields that $\phi_{A_n+B_n}(t) \to e^{-\frac{t^2}{2}}$, which implies 
	\cref{eq:M2}. \hfill $\blacksquare$ \\
	
	\noindent Following the statement of \cref{AN-iu} in the main text, we remarked that the form of the result provided included the condition that $\frac{s}{n}\frac{\zeta_s}{s\zeta_{1,\omega}}\to 0$, which thus implies that the complete analogue of the incomplete U-statistic is also asymptotically normal, but that such a condition is not necessary.  We elaborate on this point here by providing an alternative result.
	
	\begin{theorem}\label{thm:app}
		
		Let $Z_1,\dots, Z_n$ be i.i.d.\ from $F_Z$  and $U_{n,s,N,\omega}$ be a generalized incomplete U-statistic with kernel $h=h(Z_1,\dots,Z_{s};\omega)$. Let $\theta=\E[h]$ and $\zeta_s=\V(h)$. Suppose that $\E[|h-\theta|^{2k} ]/\E^2[|h-\theta|^k]$ is uniformly bounded for $k=2,3$ and for $s$. Then 
		\begin{description}
			\item[1.]
			$U_{n,s,N,\omega} -\theta = \frac{N}{\hat{N}}(A_n+B_n)$, where  $B_n= N^{-1}\sum_{(n,s)}(\rho-p)(h(Z_{i1}, \dots, Z_{is};\omega)-\theta)$ and $A_n= U_{n,s,\omega}-\theta$,    If  $s/n\to 0$ and $N\to \infty$ with $p=N/{n\choose s}\to 0$, then
			\begin{equation}
				\label{eq:M2-00}
				\frac{B_n}{\sqrt{\zeta_s/N}} \rightsquigarrow  N(0,1).
			\end{equation}
			
			\item[2.]
			In addition to the conditions in $\mathrm{\mathbf{1}}$,  If $\V(A_n) / \V(B_n)\to 0$, then 
			\begin{equation}
				\label{eq:M2-0}
				\frac{U_{n,s,N,\omega}-\theta}{\sqrt{\zeta_s/N}} \rightsquigarrow N(0,1).
			\end{equation}
		\end{description}
	\end{theorem}
	\begin{proof}
		
		For \textbf{1}, without loss of generality, let $\theta=0$. Observe that	\[ 
		\begin{aligned}
			U_{n,s,N,\omega} 
			& = \frac{1}{\hat{N}}\sum_{(n,s)}\rho h(Z_{i1},\ldots, Z_{is}; \omega)\\
			& =   \frac{1}{\hat{N}}\left[\sum_{(n,s)}(\rho-p)h(Z_{i1},\ldots, Z_{is}; \omega) + U_{n,s,\omega}\right]\\
			& =\frac{N}{\hat{N}}\left[B_n + A_n\right],
		\end{aligned} \]
		where $A_n$ and $B_n$ are uncorrelated, and $\V(B_n) = (1-p)\zeta_s/N:=d^2_{n,S,N}$.  
		First, $U_2 = {n\choose s}^{-1}\sum_{(n,s)}h^2(Z_{i1},\dots, Z_{is};\omega)$ is a complete U-statistic with kernel $h^2$. For any $\epsilon>0$, by Chebyshev's inequality, we have 
		\[
		\P\left\{\left|U_2-\E[h^2]\right|\geq \epsilon  \E[h^2]\right\}\leq \frac{s}{n}\frac{\E[h^4]}{\epsilon^2\E^2[h^2]}\leq \frac{s}{n}\frac{C}{\epsilon^2},
		\]
		which indicates that  $U_2/\zeta_s(=U_2/\E[|h|^2])\xrightarrow[]{p}1$. Let $D_2 =\left\{\zeta_s^{-1}U_2\in [1-\delta,1+\delta]\right\}$. Thus for any $\delta, \epsilon>0$, for $n$ sufficiently large, $\P(D_2) \geq 1- \epsilon$. Let $\hat{d}_{n,s,N} = \left[(1-p)U_2/N\right]^{1/2}$. Then 
		\begin{equation*}
			\begin{aligned}
				\phi_{B_n}(t) 
				& = \E\left[ \exp\left( it   B_n/d_{n,s,N}\right) \right] \\
				& = \E\left[\E\left[ \exp\left( it B_n/d_{n,s,N}\right)\mid Z_1,\dots, Z_n;\omega\right]\right] \\
				& =  \E\left[\E\left[ \exp\left( it  (\zeta_s^{-1}U_2)^{1/2} \cdot B_n/\hat{d}_{n,s,N} \right)\mid Z_1,\dots, Z_n;\omega\right]\right].
			\end{aligned}
		\end{equation*}
		$B_n/\hat{d}_{n,s,N} \mid Z_1,Z_2,\dots, Z_n;\omega$ is a sum of independent random variables and thus in order to establish asymptotic normality it suffices to check the Lyapounov's condition.  We have
		\begin{equation*}
			\mathcal{L} = \frac{{n\choose s}^{-1}\sum_{(n,s)}|h|^3}{\left({n\choose s}^{-1}\sum_{(n,s)}|h|^2\right)^{3/2}}  \frac{1-2p+2p^2}{\sqrt{N(1-p)}} =  \frac{1-2p+2p^2}{\sqrt{N(1-p)}} \frac{U_3}{U_2^{3/2}}
		\end{equation*}
		where $U_3 = {n\choose s}^{-1}\sum_{(n,s)} |h(Z_{i1},\dots, Z_{is};\omega)|^3$. Since $\E[h^6]/\E^2[|h|^3]\leq C$, by a similar argument as for $U_2$, $D_3=\left\{U_3/\E[|h|^3] \in [1-\delta, 1+\delta]\right\}$ holds with probability at least $1-\epsilon$.  Then as $N\to \infty$ and $p\to 0$,  
		\begin{equation*}
			\mathcal{L} \leq \frac{1+\delta}{(1-\delta)^{3/2}} \frac{\E[|h|^3]}{\E^{3/2}[|h|^2]} \frac{1-2p+2p^2}{\sqrt{N(1-p)}} \to 0
		\end{equation*}
		uniformly with respect to $Z_1,\dots, Z_n$ and $\omega$ over $D_2\cap D_3$. Note that distance between the characteristic function $B_n/\hat{d}_{n,s,N}$ and a standard normal distribution can be controlled by $\mathcal{L}$.  Thus  with probability at least $1-2\epsilon$, 
		\[
		\E \left[\exp\left(iuB_n/(\hat{d}_{n,s,N})\right)\mid Z_1,\dots,Z_n;\omega\right] \to e^{-\frac{u^2}{2}}
		\]
		uniformly over any finite interval  of $u$ and uniformly with respect to $Z_1,\dots, Z_n$ and $\omega$ over $D_2\cap D_3$.  Letting the interval be $[t\sqrt{(1-\delta)}, t\sqrt{(1+\delta)}]$,  we have 
		\begin{equation*}
			\left| \E\left[ \exp\left( it  (\zeta_s^{-1}U_2)^{1/2} \cdot B_n/(\hat{d}_{n,s,N}) \right)\mid Z_1,\dots, Z_n;\omega\right]- e^{-\frac{t^2}{2} \zeta_s^{-1} U_2}\right| \leq \epsilon
		\end{equation*}
		over $D_2\cap D_3$ for $N$ sufficiently large. Therefore, 
		\begin{equation*}
			\begin{aligned}
				\left|\phi_{B_n}(t) -e^{-\frac{t^2}{2}}\right|
				&\leq   \E\left[\left|\E\left[ \exp\left( it  (\zeta_s^{-1}U_2)^{1/2} \cdot B_n/\hat{d}_{n,s,N} \right)\mid Z_1,\dots, Z_n;\omega\right]-e^{-t^2/2}\right|\right] \\
				& \leq  \E\left[\left( \left| \E\left[\exp\left( it  (\zeta_s^{-1}U_2)^{1/2} \cdot B_n/\hat{d}_{n,s,N} \right)\mid Z_1,\dots, Z_n;\omega\right] - e^{-\frac{t^2}{2} \zeta_s^{-1}U_2}\right|\right.\right.  \\
				& ~~~~ + \left.\left.\left|e^{-\frac{t^2}{2} \zeta_s^{-1}U_2} - e^{-\frac{t^2}{2}}\right|\right)
				1_{D_2\cap D_3} \right] + 4\epsilon  \\
				& \leq \E\left[\left(\epsilon +\left|e^{-\frac{t^2}{2}\zeta_s^{-1}U_2} - e^{-\frac{t^2}{2}}\right|\right)
				1_{D_2\cap D_3}  \right] +4 \epsilon \\
				& \leq (\epsilon + \delta') + 4\epsilon.
			\end{aligned}
		\end{equation*}
		Since $\epsilon$ and $\delta$ can be arbitrarily small, we have $\phi_{B_n}(t)\to e^{-\frac{t^2}{2}}$ and thus $B_n/d_{n,s,N}\rightsquigarrow N(0,1)$, which implies \cref{eq:M2-00} since $N/{n\choose s}\to 0$.
		
		For \textbf{2}, $\V(A_n)/\V(B_n)\to 0$ implies that  $\V(A_n/{\sqrt{\zeta_s/N}}) = o(1)$. Thus, we have
		\begin{equation*}
			\frac{U_{n,s,N,\omega}}{\sqrt{\zeta_s/N}} = \frac{B_n}{\sqrt{\zeta_s/N}} + o_p(1),
		\end{equation*}
		which implies  \cref{eq:M2-0} by applying Slutsky's theorem. 
	\end{proof}
	
	\noindent Part \textbf{1} of \cref{thm:app} gives that $B_n$, the difference between the incomplete and complete generalized U-statistics, is asymptotically normal under quite weak conditions so long as the number of subsamples $N$ grows slower than ${n\choose s}$.  In particular, no specialized conditions on the resulting variance or variance ratio are required.  In Part \textbf{2}, to establish asymptotic normality of the generalized incomplete U-statistic itself, we do not require the original condition on the variance ratio given in \cref{AN-iu}, though we do require that $\V(A_n) /\V(B_n)\to 0$.  Such a condition remains difficult to verify for general kernels, but can always be satisfied, for example, by taking $N=o(n/s)$.  Indeed, note that
	\begin{equation*}
		\frac{s^2}{n}\zeta_{1,\omega} \leq \V(A_n) \leq \frac{s}{n}\zeta_s,
	\end{equation*}
	and $\V(B_n) = (1-p)\zeta_s/N$, thus  a sufficient condition for  $(\zeta_s/N)^{-1/2}A_n = o_p(1)$ to hold is letting $N = o(n/s)$. Thus, asymptotic normality for incomplete U-statistics can be established without requiring normality of the complete version and in particular, without requiring the specialized condition on the variance ratio discussed at length in \cref{normality}.

	\subsection{Simple Base Learner Variance Ratios}
	\label{app:B-VR}
	
	We now provide explicit calculations for the variance ratios corresponding to the various base learners discussed in \cref{normality}.  We begin with simple examples where base learners take the form of a sample mean, sample variance, and least squares estimators.
	\begin{example}[sample mean]
		\label{ex1}
		
		Suppose  $Z_1, \dots, Z_s$ are i.i.d.\ random variables with mean $\mu$ and let $h=\bar{Z}$, then 
		\begin{equation}
			\label{eq:mean}
			\frac{\zeta_s}{s\zeta_1} = \frac{ \V \left(\sum_{i=1}^s Z_i/s\right)}{s \V \left({Z_1}/{s}+ {(s-1)\mu}/{s}\right)} = 1.
		\end{equation}
		
	\end{example}
	\cref{eq:mean} holds for any estimators that can be written as a sum of i.i.d.\ random variables. In such cases, since $\hat{h} = h$, nothing is lost after projecting.
	
	\begin{example}[sample variance]
		\label{ex2}
		
		Suppose $Z_1,\dots, Z_s$ are i.i.d.\ random variables with variance $\sigma^2$ and fourth central moment $\mu_4$ and consider the sample variance $h = \binom{n}{2}^{-1}\sum_{i< j}(Z_i-Z_j)^2$. Then as $s\to \infty$, 
		\begin{equation*}
			\label{eq:variance}
			\frac{\zeta_s}{s\zeta_1} = 1 + \frac{2}{(s-1)}\cdot \frac{\sigma^4}{\mu_4-\sigma^4} \to 1.
		\end{equation*}
		
	\end{example}
	The kernel $h$ can be written as $h =\sum_{i=1}^s \frac{(Z_i-\bar{Z})^2}{s-1}
	\approx \sum_{i=1}^s  \frac{(Z_i^2-\mu^2)}{s-1}$. Since $\bar{Z}$ is much more stable than $Z_i$, $h$ is  close to a sum of i.i.d random variables.

	\begin{example}[OLS estimator]
		\label{ex3}
		
		Let $Z_1, \dots , Z_s$ denote i.i.d.\ pairs of random variables $(X_i,Y_i)$ and $Y_i = X_i^T\beta+\epsilon_i$. Suppose  that $\epsilon_i$ has mean $0$ and variance $\sigma^2$, and $\epsilon_i$ is independent of $X_i$.  Let $h = (X^TX)^{-1}X^TY$ be the ordinary least squares (OLS) estimator of $\beta$. Then as $s \to \infty$,
		\begin{equation*}
			\label{eq:lm}
			(s\zeta_1)^{-1}\zeta_s \to I,
		\end{equation*}
		where $I$ is the identity matrix.
		
	\end{example}
	
	\begin{proof}
		Let $\hat{\beta} = GX^TY$ be the OLS estimator, where $G=(X^TX)^{-1}$, then  $\zeta_s=\V(\hat{\beta}) = \E[G]\sigma^2$.
		Since $X_i$ and $\epsilon_i$ are independent, we have $\E[\hat{\beta}\mid X_1,Y_1] =\beta + \E_1[G]X_1\cdot \epsilon$, where $\E_1$ takes the expectation conditioning on $X_1$. Then,
		\begin{equation*}
			\zeta_1 = \V\left(\E[\hat{\beta} \mid X_1,Y_1]\right) = \E[\E_1[G]X_1X_1^T\E_1[G]]\sigma^2.
		\end{equation*}
		According to law of large numbers, $\frac{1}{s}\sum_{i=1}^s X_iX_i^T\xrightarrow[]{a.s.}\Sigma$ as $s\to \infty$ where $ \E[X_i X_i^T] = \Sigma$.
		Letting $\Sigma^{-1} = \Omega$, we have
		\[ 
		s\cdot \E[(X^TX)^{-1}] = \E\left[\left(\frac{1}{s}\sum_{i=1}^s X_iX_i^T\right)^{-1}\right]\to \Omega,
		\]
		and hence $s\zeta_s \to \Omega$. Furthermore, we have
		\[
		\begin{aligned}
			sG \mid X_1 
			& = \left(\frac{1}{s}\left[X_1X_1^T+ \sum_{i\neq 1} X_iX_i^T\right]\right)^{-1} \Bigg| \, X_1\xrightarrow[]{a.s.} \Omega.
		\end{aligned}
		\]
		Thus,
		\[
		\begin{aligned}
			s^2\cdot \zeta_1/\sigma^2 
			& = \E\left[\E_1[G]X_1X_1^T\E_1[G]\right]\\
			& = \E\left[\E_1[sG]\cdot X_1X_1^T\cdot \E_1[sG]\right]\\
			& \to \Omega\cdot \Sigma \cdot \Omega \\
			& = \Omega,
		\end{aligned}\]
		and hence $\zeta_1$ is of order $s^{-2}$. Therefore, $(s\zeta_1)^{-1}\zeta_s \to I$.
	\end{proof}
	
	\noindent Here again, note that $h =\sum_{i=1}^s (X^TX)^{-1} X_iY_i $, which is still close to a sum of i.i.d.\ random variables.  These three examples suggest that perhaps for many common base learners, $\zeta_s$ is of order $s^{-1}$ and $\zeta_1$ of order $s^{-2}$; essentially, each individual observation explains roughly $s^{-1}$ times the variance of the base learner.  \\

	
	
	\noindent \textbf{Proof of \cref{knn}: } Denote the kNN estimator at $x$ as $\varphi(x)$. Let $A_i$ denote the event that $X_1$ is the $i$th closest point to the target point $x$ and $B = \cup_{i=1}^k A_i$. First, by the continuity of $f$ at $x$, we have $
	\V(\varphi(x)) \to {\sigma^2}/{k}$ as $s \to \infty$.  Let $X_1^*,\ldots, X_k^*$ be the $k$-NNs of $x$.  Then 
	\[
	\begin{aligned}
		\E[\varphi(x)\mid X_1,Y_1] 
		&= \frac{1}{k}\E\left[ \one_{B}\left[Y_1+\sum_{i=2}^k Y_i^*\right]  + \one_{B^c}\left[\sum_{i=1}^k Y_i^*\right]   \mid X_1,Y_1\right] \\
		& = \frac{\epsilon_1}{k} \E[\one_B\mid X_1] + \frac{1}{k}\left[\sum_{i=1}^k f(X_i^*)\mid X_1\right],
	\end{aligned}
	\]
	and thus $\V(\E[\varphi(x)|X_1,Y_1]) \geq \frac{\sigma^2}{k^2}\E[\P^2(B|X_1)] $.
	Next,
	\begin{align*}
		\E[\P^2(B|X_1)]
		& = \E [(\sum_{i=1}^k\P(A_i \mid X_1))^2] \\
		& = \sum_{i=1}^k\sum_{j=1}^k \E\left[\P(A_i\mid X_1)\P(A_j\mid X_1)\right] \\
		& = \frac{1}{2s-1}\sum_{i=1}^k\sum_{j=1}^k \frac{{s-1\choose i-1} {s-1\choose j-1}}{{2s-2\choose i+j-2}} \\
		& = \frac{ V(k,s)}{2s-1} \label{cc1} \stepcounter{equation}\tag{\theequation},
	\end{align*}
	where  $V(k,s) = \sum_{i=0}^{k-1}\sum_{j=0}^{k-1} \left[ \frac{{s-1\choose i} {s-1\choose j}}{{2s-2\choose i+j}}  \right]$.   We have 
	\[
	\begin{aligned}
		V(k) &= \lim_{s\to \infty}V(k,s)\\
		&= \sum_{i=0}^{k-1}\sum_{j=0}^{k-1} \frac{(i+j)!}{i!j!}\frac{1}{2^{i+j}}\\
		&= \sum_{c=0}^{2k-2}\sum_{i+j=c,0\leq i,j\leq k-1}\frac{(i+j)!}{i!j!}\frac{1}{2^{i+j}} .
	\end{aligned}\]
	We can obtain $k\leq V(k)< 2k-1$ by simply observing that
	\begin{equation*}
		\sum_{c=0}^{k-1}\sum_{i+j=c}\frac{c!}{i!j!}\frac{1}{2^{c}}\leq V(k) \leq  \sum_{c=0}^{2k-2}\sum_{i+j=c}\frac{c!}{i!j!}\frac{1}{2^{c}}.
	\end{equation*}
	Thus,
	\begin{equation}
		\begin{aligned}
			\limsup_{s\to \infty } \frac{\zeta_s}{s\zeta_1}  
			& = \limsup_{s\to \infty} \frac{\V(\varphi(x))}{s\V(\E[\varphi(x)|X_1,Y_1])}\\ 
			& \leq \limsup_{s\to \infty} \frac{2s-1}{s}\frac{k}{V(k,s)} \\
			& = c(k),
		\end{aligned}
	\end{equation}
	and $1 < c(k) <= 2$. \hfill $\blacksquare$ \\
	
	Note that \cref{knn} holds without imposing any conditions on the regression function $f$ or the distribution of $X$. To see why, note from the proof that both $\zeta_s$ and $\zeta_1$ can each be decomposed into two terms, one of which comes from the variation of the regression function while the other is due to the variation of the noise. Here, since $k$ is fixed, the term involving the variation in the regression function is small relative to the noise term for large $s$.  When $k$ grows with $s$, more care must be taken in assessing the contribution of  the regression function. \cite{Biau2015} (Theorem 15.3) discuss the convergence rate of the variance of kNN estimators.  This result could potentially enable results to be established for more general nearest neighbor estimators where the number of neighbors $k$ is permitted to grow with $s$, though we do not explore this further here. \\

	
	\noindent \textbf{Proof of \cref{ls}: } First let $\tilde{\varphi} = \sum_{i=1}^s w(i,x,\mathbf{X})f(X_i)$ and note that 
	\begin{equation}
		\label{bound-1}
		\begin{aligned}
			\V(\varphi)  & = s\E[w^2(1,x,\mathbf{X})]  \sigma^2  +\V\left(\tilde{\varphi}\right) \\
			& \leq \sigma^2+ ||f||_\infty/4.
		\end{aligned}
	\end{equation}
	Next, $	\E[\varphi \mid X_1,\epsilon_1] 
	= \E\left[\tilde{\varphi}\mid X_1\right] + \epsilon_1 \E\left[w(1,x,\mathbf{X}) \mid X_1\right]$,
	and thus
	\begin{equation}
		\label{bound-2}
		\begin{aligned}
			\V(\E[\varphi \mid X_1,\epsilon_1]) 
			& =\V\left(\E \left[\tilde{\varphi} \mid X_1\right]\right)+ \sigma^2\E\left[\E^2\left[w(1,x,\mathbf{X}) \mid  X_1\right]\right] \\
			& \geq \sigma^2 \E\left[\E^2[w(1,x,\mathbf{X}) \mid X_1]\right]\\
			& \geq \sigma^2 \E^2[w(1,x,\mathbf{X})] \\
			& = \sigma^2/s^2.
		\end{aligned}
	\end{equation}
	Therefore, 
	\begin{align*}
		\limsup_{s\to \infty}\frac{\zeta_s}{s\zeta_1} \frac{1}{s}
		&= \limsup_{s\to  \infty}\frac{s\sigma^2 \E[w^2(1,x,\mathbf{X})] + \V(\tilde{\varphi})}{s^2   \left(\sigma^2 \E[\E^2[w(1,x,\mathbf{X})) \mid X_1]] + \V(\E[\tilde{\varphi} \mid X_1])\right)  } \label{overall-vr}    \stepcounter{equation}\tag{\theequation} \\
		&  \leq \frac{\sigma^2 + ||f||_\infty^2/4}{\sigma^2} \\
		& <\infty. \hspace{100mm} \blacksquare \\
	\end{align*}
	\noindent We emphasize that the inequalities in \cref{bound-1} and \cref{bound-2} are generally quite loose in order to cover the worst case scenario. As seen in \cref{knn}, the order of $ \zeta_1$ can indeed be $s^{-1}$ rather than $s^{-2}$. Nonetheless, \cref{ls} indicates that $s/\sqrt{n}\to 0$ is sufficient to ensure that $\frac{s}{n} \frac{\zeta_s}{s\zeta_1}\to 0$.

	\subsection{Variance Ratios for RP Trees }
	\label{app:B-RP}
	Now, we turn to the analysis for RP trees, which form predictions by taking a sample uniformly at random from the potential nearest neighbors and averaging the corresponding response values.  We begin with a simpler result for base learners that take a naive random average across $k$ response values selected uniformly at random from the entire dataset.
	
	\begin{example}[naive random average]
		\label{ex6}
		Let $Z_1,\dots, Z_s$ denote i.i.d.\ pairs of random variables $(X_i,Y_i)$.  For any target point $x$, let $\varphi(x)$ denote the estimator that forms a prediction by simply selecting $k$ sample points uniformly at random (without replacement) and averages the selected response values, so that we can write $\varphi(x) = \frac{1}{k} \sum_{i=1}^s  \xi_iY_i$, where
		\[
		\xi_i = \begin{cases}
			1, & i^{th}~ \mathrm{ sample ~ is ~ selected}\\
			0, & i^{th}~\mathrm{ sample ~ is ~ not ~  selected}.
		\end{cases}
		\]
		Then $ \zeta_s=  \V(Y_1)/k$ and $ \zeta_{1,\omega}= (\V(Y_1))/s^2$,  so that $ {\zeta_s}/{s\zeta_{1,\omega}}= s/ k$.
	\end{example}
	
	Note in the above example that when $k$ is fixed, $s/\sqrt{n}\to 0$ is sufficient to ensure that $\smash{\frac{s}{n} \frac{\zeta_s}{s\zeta_{1,\omega}}\to 0}$.  However, when $k$ is assumed to grow with $n$, the subsample size $s$ can grow more quickly.  In the adaptive case, where $w(i,x,\mathbf{X})$ may depend on $\{Y_i\}_{i=1}^s$, tree estimators with small terminal node sizes may look less like a linear statistic and in turn may have a larger variance ratio. However, as discussed, for non-adaptive estimators like kNN, the ratio is bounded by a constant. In this way, well-behaved tree predictors can be seen as similar to kNN and are still more easily controlled than RP trees. 
	
	We turn now to the proving Proposition \ref{rptree}, namely that for base learners that are RP trees, 
	\[ 
	\limsup_{s\to \infty}\frac{\zeta_s/{s\zeta_{1,\omega}}}{(\log s)^{2d-2}}<\infty .
	\]
	
	~\\
	
	\noindent \textbf{Proof of \cref{rptree}: } Denote the RP tree by $T$ and the set of $k$-$\mathrm{PNNs}$ by $\Xi$.  We have
	\[ 
	\V(T)  = \V(\tilde{T}) + {\sigma^2}/{k}   \leq \sigma^2/k + ||f||^2_\infty/4.
	\]
	where $\tilde{T}$ is the RP tree prediction in the noiseless case. Let $ |\Xi|$ denote the number of $k$-$\mathrm{PNNs}$ of $x$, then 
	\[ 
	\E[T\mid Z_1] = \epsilon_1 \E\left[\frac{1}{|\Xi|}\one_{X_1\in \Xi} \mid X_1\right] + \E\left[\tilde{T}  \mid X_1\right].
	\]
	Since $\Xi$ is independent of $\epsilon$, we have $\V(\E[T\mid Z_1])  \geq \sigma^2 \E\left[\E^2[S_1 \mid X_1]\right]$,  where $S_1 = \frac{1}{|\Xi|}\one_{X_1\in \Xi}$. Note that $\E[S_1] = \E\left[\sum_{i=1}^s S_i\right]/s =1/{s}$
	, and thus we have $\V(\E[T\mid Z_1])\geq \sigma^2/s^2$. Furthermore,  we have
	\begin{equation} 
		\label{I0}
		\begin{aligned}
			\E[S_1|Z_1] 
			& =  \P_1(X_1\in \Xi) \E_1 \left[\frac{1}{|\Xi|} \mid {X_1\in \Xi}\right]\\
			& =  \sum_{i=0}^{k-1}{s-1\choose i}u^i(1-u)^{s-1-i}  \, \cdot \, \E_1 \left[\frac{1}{|\Xi|} \mid {X_1\in \Xi}\right] \\
			& = \mathrm{I} \cdot  \mathrm{II},
		\end{aligned}
	\end{equation}
	where  $\P_1 (\cdot)= \P(\cdot \mid X_1), \E_1 = \E(\cdot \mid X_1)$, $u=\P_1(X_i\in R)$ and $R$ is the hyperrectangle defined by $x$ and $X_1$. 
	Conditioning on $X_1\in \Xi$, define the conditional probability function of $(X_2,\ldots, X_s)$ - ${\P}_1(\cdot \mid X_1 \in \Xi)$ as $\tilde{\P}_1$.  
	For $\mathrm{I}$, 
	\[
	\begin{aligned}
		\mathrm{I}^2 
		& = \left[\sum_{i=0}^{k-1}{s-1\choose i}u^i(1-u)^{s-1-i} \right]^2 \\
		& = \sum_{i=0}^{k-1}\sum_{j=0}^{k-1} {s-1\choose i}{s-1\choose j}u^{i+j}(1-u)^{2s-2-i-j} ,
	\end{aligned}
	\]
	thus 
	\begin{align*}
		\E[\mathrm{I}^2]  
		& = \sum_{i=0}^{k-1}\sum_{j=0}^{k-1} {s-1\choose i}{s-1\choose j} \E\left[ u^{i+j} (1-u)^{2s-2-i-j}\right]\\
		&  = \sum_{i=0}^{k-1}\sum_{j=0}^{k-1} \frac{{s-1\choose i}{s-1\choose j}}{{2s-2 \choose i+j}} \E\left[ \frac{u^{i+j} (1-u)^{2s-2-i-j}}{\mathrm{B}(i+j+1,2s-1-i-j)}\right] \\
		& =  \frac{1}{2s-1} \cdot \sum_{i=0}^{k-1}\sum_{j=0}^{k-1} \frac{{s-1\choose i}{s-1\choose j}}{{2s-2 \choose i+j}} G(i,j)\stepcounter{equation}\tag{\theequation}\label{cc2},
	\end{align*}
	where $u = \P_1(X_i\in R)\in (0,1)$.  If $u\sim \mathrm{Uniform}(0,1)$, then $G(i,j)=1$ and \eqref{cc2} reduces to \eqref{cc1}. Let the probability density function  of $u$ be $p(u)$, and the probability density function of beta distribution with shape parameters $\alpha$ and $\beta$ be $g(u,\alpha,\beta)$, then $G(i,j)= \int_0^1 g(u,\alpha,\beta)p(u)\,du$, where $\alpha=i+j, \beta=s-1-\alpha$. Since $i+j<=2k-2$, when $s\to \infty$, $g(u,\alpha,\beta)$ is almost singular at $u=0$. Moreover, we can find that at around $u=0$, $p(u)\geq 1$.	Thus, there exists some $c_1>0$ such that $G(i,j)\geq c_1$. Therefore, we have
	\begin{equation}
		\label{I}
		\E[\mathrm{I}^2] \geq \frac{c_1}{2s-1}V(k,s) ,\quad V(k,s) =   \sum_{i=0}^{k-1}\sum_{j=0}^{k-1} \left[\frac{{s-1\choose i}{s-1\choose j}}{{2s-2 \choose i+j}} \right].  
	\end{equation}
	
	\noindent For $\mathrm{II}$,  by Jensen's Inequality, we have $	\mathrm{II} = \tilde{\E}_1\left[{1}/{|\Xi|}\right] \geq {1}/{\tilde{\E}_1[|\Xi|]}$, and then
	\begin{equation}
		\label{I-II}
		\begin{aligned}
			\E[\E^2[S_1|X_1]]
			&\geq  \E\left[\mathrm{I}^2 \cdot \frac{1}{\tilde{\E}_1^2|\Xi|}\right].
		\end{aligned}
	\end{equation}
	
	\noindent$\tilde{\E}_1|\Xi|$ is just the expected number of $k$-$\mathrm{PNNs}$  conditioning on $X_1\in \Xi$,  or equivalently given that there are at most $k-1$ sample points in $R$.  \cite{Lin2006} showed that $\E[|\Xi]|$ is of order $k(\log s)^{p-1}$ when the probability density function of the features is bounded away from $0$ and $\infty$ in $[0,1]^p$. Since $k$-$\mathrm{PNN}$ depends only on the relative distance,  it can be shown that there exists some $c_2 >0$ such that 
	\begin{equation}
		\label{II}
		\tilde{\E}_1 [|\Xi|] \leq c_2 \E[|\Xi|], \quad \text{for } X_1 \in [0,1]^p.
	\end{equation}
	Note that $V(k,s) \geq k$. Combining \cref{I}, \cref{I-II} and \cref{II}, we have 
	\begin{equation}
		\label{rp}
		\begin{aligned}
			\limsup_{s\to \infty}\frac{\zeta_s/{s\zeta_{1,\omega}}}{(\log s)^{2p-2}} \leq \limsup_{s\to \infty}\frac{(\sigma^2/k + ||f||_\infty^2/4) (\log s)^{2-2p}}{ s\left(  c_1\frac{V(k,s)}{2s-1} \cdot c_2^{-2} (k(\log s)^{p-1} )^{-2}\cdot \sigma^2\right)} <\infty,
		\end{aligned}
	\end{equation}
	thus achieving what was claimed in Proposition \ref{rptree}.  \hfill $\blacksquare$ \\

	\subsection{Variance Ratio for non-adaptive $k$-trees}

	\label{app: B-AD}

	Finally, we turn to the analysis for non-adaptive $k$-trees, which form predictions by taking a sample non-adaptively from the potential nearest neighbors and averaging the corresponding response values. 
	\\
	\\
	\noindent \textbf{Proof of Proposition \ref{non-adaptive}: } Denote a non-adaptive $k$-tree by $T= T(Z_1,\dots, Z_s;\omega)$ and the set of $2k$-$\mathrm{PNNs}$ by $\Xi$.  We have
	\[ 
	\V(T)  = \V(\tilde{T}) + {\sigma^2}/{k}   \leq \sigma^2/k + ||f||^2_\infty/4,
	\]
	where $\tilde{T}$ is the tree prediction in the noiseless case. Let $\mathrm{L}(x;\omega)$ be the terminal node that $x$ belongs to, then 
	$
	\E[T\mid Z_1] = \epsilon_1 \E\left[\frac{1}{|\mathrm{L}(x;\omega)|}\one_{X_1\in \mathrm{L}(x;\omega)} \mid X_1\right] + \E\left[\tilde{T}  \mid X_1\right]
	$.
	Since $\mathrm{L}(x;\omega)$ is independent of $\epsilon$, we have 
	\begin{equation}
		\begin{aligned}
			\V(\E[T\mid Z_1])   \geq \sigma^2 \E\left[ \E^2\left[\frac{1}{|\mathrm{L}(x;\omega)|}\one_{X_1\in \mathrm{L}(x;\omega)} \mid X_1\right] \right]
			= \sigma^2 \E[S_1^2],
		\end{aligned}
	\end{equation}
	where $S_1 =\E[ \frac{1}{|\mathrm{L}(x;\omega)|}\one_{X_1\in \mathrm{L}(x;\omega)}\mid X_1]$. Note that $\E[S_1] = \E\left[\sum_{i=1}^s S_i\right]/s =1/{s}$, and thus we have $\V(\E[T\mid Z_1])\geq \sigma^2/s^2$. Furthermore, since $\mathrm{L}(x;\omega)\leq 2k-1$, then $\mathrm{L}(x;\omega)\subset 2k$-$\mathrm{PNNs}$ and  we have
	\begin{equation}
		\label{random_not}
		\begin{aligned}
			S_1 & = \E[ \frac{1}{|\mathrm{L}(x;\omega)|}\one_{X_1\in \mathrm{L}(x;\omega)}\mid X_1] \\
			&   \leq \frac{1}{k}\E[\one_{X_1\in \Xi}\mid X_1] = \frac{1}{k}\cdot \mathrm{I},
		\end{aligned}
	\end{equation}
	where $\mathrm{I0}$ is the same as in \cref{I}. We get rid of the extra randomness $\omega$ here.
	From the proof of \cref{rptree}, we know  $\E[\mathrm{I}^2] = O(1/s)$ and according to \cite{Lin2006}, we have $\E[\mathrm{I}] = O(2k(\log s)^{p-1}/s)$. Note that for any random variable $X$ and any real number $x$, we always have 
	\begin{equation}
		\label{trick}
		\begin{aligned}
			\E[X^2] & \geq \E[X^2; X\geq x] \\
			&  = \Pr(X\geq x)  \E[X^2\mid  X \geq x] \\
			& \geq  \Pr(X \geq x)   \E^2[X\mid  X\geq x]\\
			& =\frac{ \E^2[X;  X\geq x]}{\Pr(X\geq x)}.
		\end{aligned}
	\end{equation}
	Now, let $x =  s^{-(1+\alpha)}$ for any $\alpha >0$, then $\Pr(\mathrm{I}\geq s^{-(1+\alpha)}) \leq \frac{ \E^2[\mathrm{I};  \mathrm{I}\geq  s^{-(1+\alpha)}] }{\E[\mathrm{I}^2]} = O({(\log s)^{2p-2}}/{s})$. Then, by applying \cref{trick} again, we have 
	\begin{equation}
		\begin{aligned}
			\E[S_1^2] & \geq \frac{\E^2[S_1;S_1\geq k^{-1}s^{-(1+\alpha)}]}{\Pr(S_1\geq k^{-1}s^{-(1+\alpha)})}.
		\end{aligned}\\
	\end{equation}
	Since $\E[S_1]= 1/s$, $\E^2[S_1;S_1\geq k^{-1}s^{-(1+\alpha)}] = O(1/s^2)$, and hence
	$\E[S_1^2]\geq  \frac{O(1/s)}{\Pr(\mathrm{I}\geq s^{-(1+\alpha))})} = O(1/s(\log s)^{2p-2})$. Thus, $\V(\E[T|Z_1])  \geq   \sigma^2 \cdot O(1/(s\log s )^{2p-2})$. Therefore, 
	\begin{equation}
		\begin{aligned}
			\limsup_{s\to \infty}\frac{\zeta_s/{s\zeta_1}}{(\log s)^{2p-2}} \leq \limsup_{s\to \infty}\frac{(\sigma^2/k + ||f||_\infty^2/4) (\log s)^{2-2p}}{ s\left( s^{-1}  ((\log s)^{p-1} )^{-2}\cdot \sigma^2\right)} <\infty,
		\end{aligned}
	\end{equation}
	thus achieving what was claimed in Proposition \ref{non-adaptive}.  \hfill $\blacksquare$ \\
	
	Note that the proof strategy here follows closely to that in \cite{Wager2018}; the main difference above is in how we obtain $\Pr(\mathrm{I}\geq s^{-(1+\alpha)})$.

	\subsection{Simulation}
	\label{app:B-SI}
	We present here a small simulation study in order to illustrate the behavior of the variance ratio of when the base learners employed are decision trees or nearest neighbor methods, as well to visually demonstrate the asymptotic normality derived in \cref{normality}. We consider the underlying regression function:
	\begin{equation}
		\label{MARS}
		y = 1.0\sin(\pi x_1 x_2) + 0.2  (x_3- 0.5)^2 + 0.5  x_4 + 0.1 x_5 + \epsilon, \quad \epsilon \sim \mathcal{N}(0,0.3),
	\end{equation}
	where $(x_1,\dots, x_5)\sim \mathcal{N}(\mu, \Sigma)$, $\mu = (0.2, 0.3, 0.2, 0.7, 0.4)$, $\Sigma_{(i,i)} = 1$ and $\Sigma_{(i,j)}=0.2$ for $1\leq i, j\leq 5 $. This function has slightly different forms, but was originally used in the MARS paper \cite{friedman1991multivariate} and has since been frequently used in more recent work exploring the behavior of random forests \citep{Biau2012,Mentch2016, JMLR:v21:19-905}. We consider the prediction at $x_0 = (0.5,0.5,0.5,0.5,.0.5)$ for the base learners. $\zeta_s$ and $\zeta_{1,\omega}$ are obtained separately. For $\zeta_s$, we generate $n_{\text{out}}$ independent samples of size $s$, build base learners on each, and then use the sample variance of the predictions as our estimate. The parameter $\zeta_{1,\omega} = \V(\E[h \mid Z_1])$ is estimated in two steps. We begin by sampling one observation $Z_1$, which we refer to as the initial fixed point. We then generate several samples of size $s-1$ and combine them with $Z_1$ to build a base learner and then record the mean of the predictions at $x_0$. Let $n_\text{in}$ denote the number of samples drawn so that this average is taken over $n_\text{in}$ predictions. We then repeat the process for $n_\text{out}$ initial sets of fixed points and take our final estimate of $\zeta_{1,\omega}$ as the variance over the $n_\text{out}$ final averages, yielding the estimator.

	It is worth pausing for a moment to note that previous work has suggested that a larger $n_{\text{in}}$ seems to provide better results even when $n_{\text{out}}$ is relative small \cite{Mentch2016}.  An explanation for this is as follows. Denote $\E[h\mid Z_1=z] = h_1(z)$ and note that the inner approximation is trying to estimate $h_1(z)$. The outer approximation is intended to estimate $\V(h_1(Z))$, but it is actually estimating $\V(\hat{h}_1(Z))$. By law of total variance, we have 
	\begin{equation}
		\begin{aligned}
			\V_n(\hat{h}_1(Z)) & = \V_n(\E[\hat{h}_1(Z)\mid Z]) + \E_n[\V(\hat{h}_1(Z)\mid Z)] \\
			& = \V_n(h_1(Z)) + \E_n[\V(\hat{h}_1(Z)\mid Z)],
		\end{aligned}
	\end{equation} 
	where $\V_n(\cdot)$ denotes sample variance and $\E_n(\cdot)$ denotes sample mean. Then 
	\begin{equation}
		\label{est-bias}
		\begin{aligned}
			\frac{\E[ \V_n(\hat{h}_1(Z))]}{\V(h_1(Z))}  
			& = 1 + \frac{\E[\V(\hat{h}_1(Z)\mid Z)]}{\V(h_1(Z))} \\
			& =  1 + \frac{\E[\V(\hat{h}_1(Z)\mid Z)]}{\zeta_{1,\omega}}.
		\end{aligned}
	\end{equation}
	On the other hand, note that $\V(\hat{h}_1(Z)\mid Z) = \V(h \mid Z_1=Z)/ n_{\text{in}} $. When $s$ is large, knowing one observation will not influence the final prediction by much and so $ \V(h(\mid Z_1=Z ) \approx \zeta_s$. Therefore, the right hand side of \cref{est-bias} is close to $1 + \frac{s}{n_{\text{in}}} \frac{\zeta_s}{s\zeta_{1,\omega}}$.
	Since $\frac{\zeta_s}{s\zeta_{1,\omega}}\geq 1$, to make $\V_n(\hat{h}_1(Z))$ less biased, we need to make $n_{\text{in}}$ large enough relative to $s$.
	\\
	\begin{figure}[t]
		\centering
		\includegraphics[width=1\textwidth]{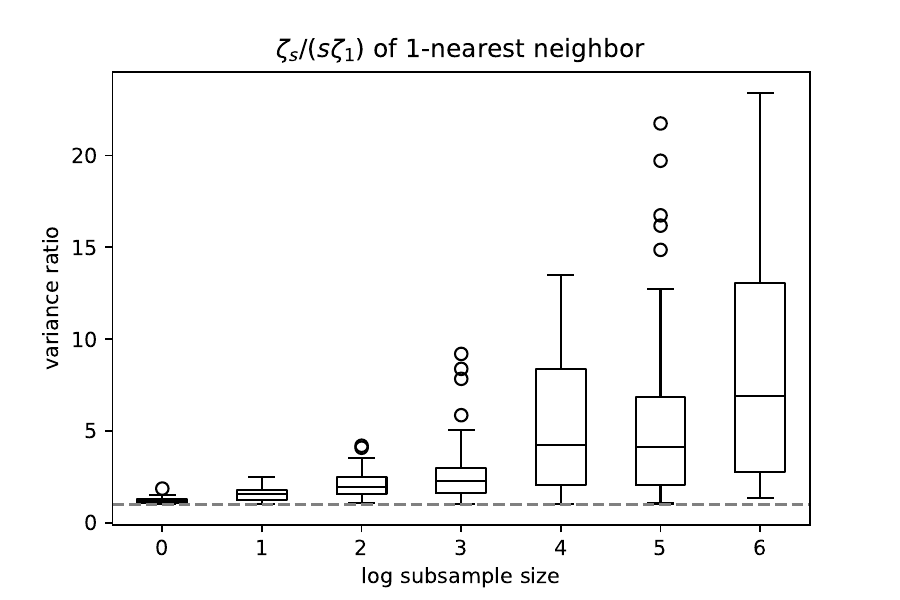}
		\caption{Variance ratio of 1-NN of different subsample size. We estimate $\zeta_s/{s\zeta_{1,\omega}}$ 30 times to obtain the boxplots. We can see that as the sample size increases, the variance ratio appears to slowly get larger.}
		\label{fig:knn_s}
	\end{figure} 
	\\
	To carry out these simulations with kNN base learners, we use the Python package \\ \texttt{sklearn.neighbors.KNeighborsRegressor} with the distance metric is set to as ``minkowski".  \cref{fig:knn_s} uses 1-nearest neighbor base learners and shows how the estimated variance ratio changes with (log) subsample size. Here $n_\text{in}=1000$ and $n_\text{out}=50$ and we obtain the boxplots of the ratio $\zeta_s/{s\zeta_1}$ for $s = 2^i$, for $i=1,\dots, 7$.  We see that as $s$ increases, so does the variance in our estimation.  This is because $\zeta_1$ is getting smaller and becoming harder to estimate.  However, if look at the median of the estimated variance ratio $\zeta_s/{s\zeta_1}$, we see that it increases as $s$ increases at a rate roughly on the order of $\log(s)$. 
	\begin{figure}[t]
		\centering
		\includegraphics[width=1\textwidth]{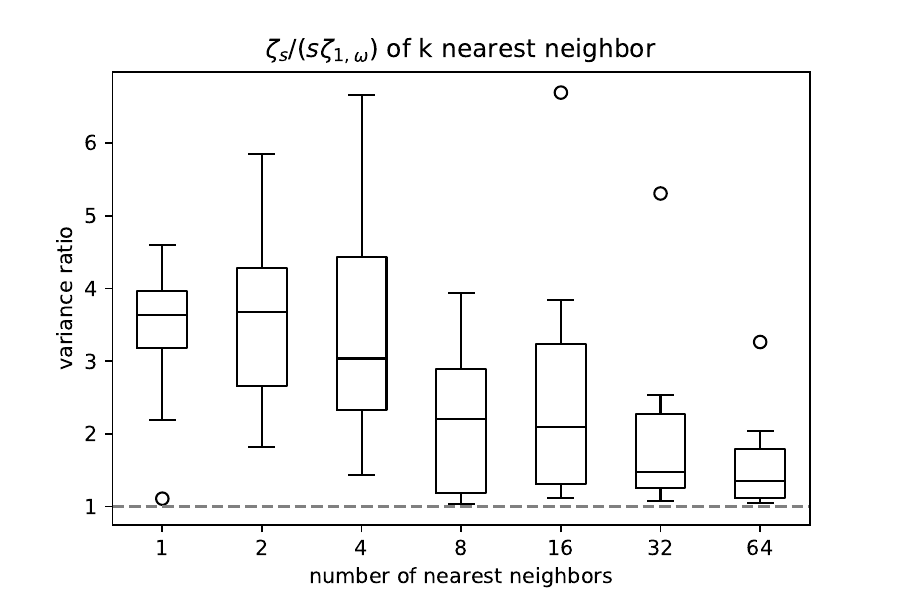}
		\caption{Variance ratio of $k$NN with differing numbers of nearest neighbors. We see that the variance ratio appears to get smaller as the number of nearest neighbors increases.}
		\label{fig:knn_k}
	\end{figure} 
	
	In \cref{fig:knn_k}, we fix the subsample size $s$ to be $300$ and consider $k$-nearest neighbor base learners with differing values of $k$, where $k=2^i$ for $i=1,\dots, 7$. Here $n_\text{in}=1000$ and $n_\text{out}=50$. As expected, when $k$ increases, $\zeta_s/{s\zeta_1}$ decreases and gets closer to $1$.  Again, as $k$ increases, $\zeta_1$ gets larger, so the estimation of $\zeta_1$ also becomes more accurate, which can be seen from the trend in the variation of the estimates.
	
	To carry the same kinds of simulations with decision trees, we use the Python package \texttt{sklearn.tree.DecisionTreeRegressor}. We let the minimum number of samples required to split an internal node be $2k$ and the minimum number of samples required in a (terminal) leaf node be $k$ so that the terminal node size is always bounded between $k$ and $2k-1$. The number of features randomly made eligible for splitting at each internal node -- often referred to as the \texttt{mtry} parameter -- is set to the square root of total number of features.
	\begin{figure}[h]
		\centering
		\includegraphics[width=1\textwidth]{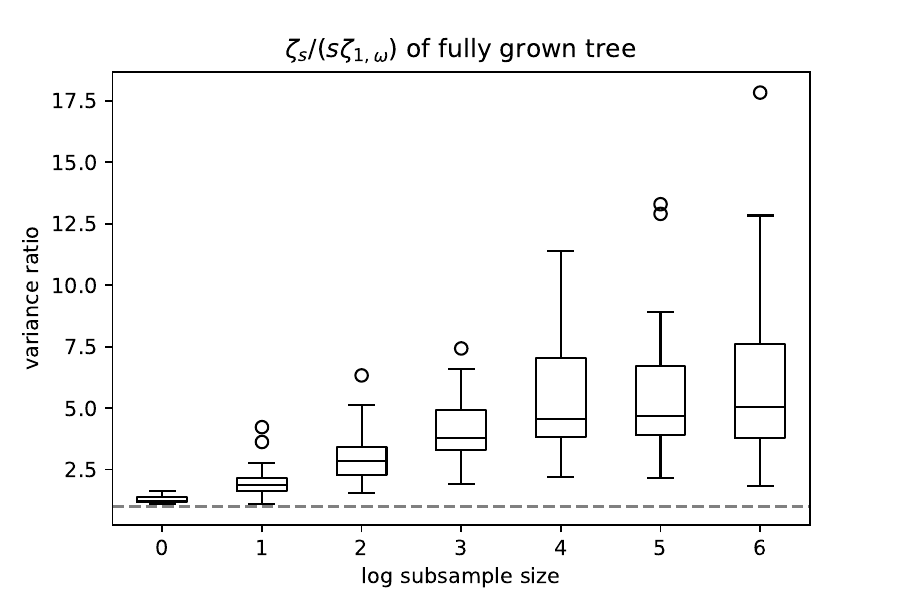}
		\caption{Variance ratio of fully grown trees built with subsamples of varying size.  We can see that as the sample size increases, the variance ratio appears to slowly get larger.}
		\label{fig:tree_s}
	\end{figure} 
	
	\cref{fig:tree_s}, utilizes fully grown trees with $n_\text{in}=1000$ and $n_\text{out}=50$ and again obtain boxplots of the estimated variance ratio $\zeta_s/{s\zeta_{1,\omega}}$ for $s = 2^i$, for $i=1,\dots, 7$.  Once again we see that as $s$ increases, the variance of our estimates becomes larger and the estimation less accurate.  We also again see that the median of the estimated variance ratio $\zeta_s/{s\zeta_{1,\omega}}$ increases as $s$ increases at a rate roughly on the order of $\log(s)$.
	\begin{figure}[h]
		\centering
		\includegraphics[width=1\textwidth]{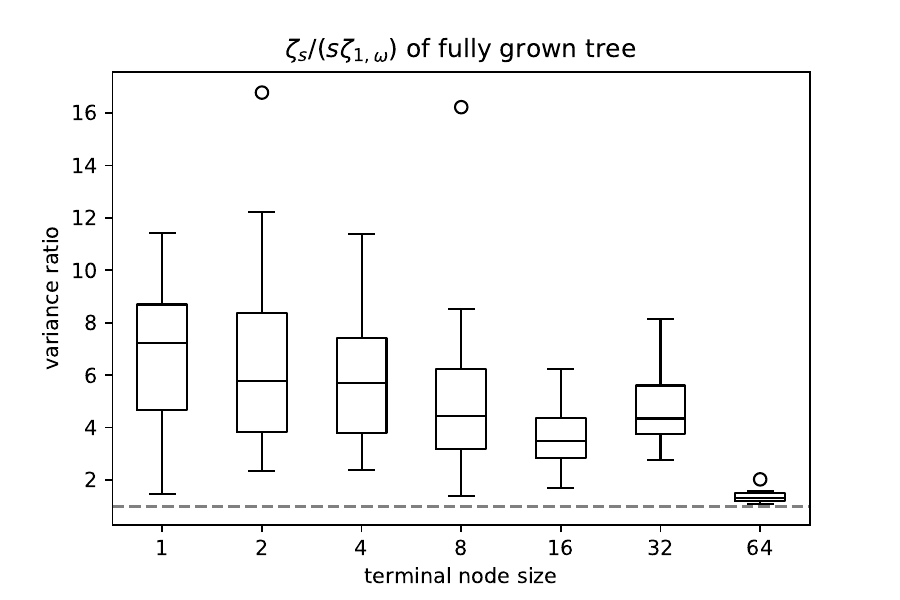}
		\caption{Variance ratio of fully grown trees with differing minimal terminal sizes. We see that the variance ratio get smaller as the minimal terminal size increases.} 
		\label{fig:tree_k}
	\end{figure}

	In \cref{fig:tree_k}, we fix the subsample size $s$ to be $100$ and consider trees with the terminal sizes ranging between $k$ and $2k-1$, where $k=2^i$ for $i=1,\dots, 7$. Compared to ensembles of $k$-NN base learners, the trees here have extra randomness through $\omega$ and so $\zeta_{1,\omega}$ will be relatively quite small. Thus, we let $s$ be $100$ here instead of $300$ in an effort to make $\zeta_{1,\omega}$ large enough for reasonably accurate estimation.  Here we let $n_\text{in}=2000$ and $n_\text{out}=50$. As $k$ increases, we can see that $\zeta_s/{s\zeta_{1,\omega}}$ decreases and gets closer to $1$, as expected.  
	
	Finally, we turn to illustrating the asymptotic normality of generalized incomplete U-statistics, where the base learners are either nearest neighbor methods (kNN) or decision trees.  Here we utilize the same prediction location and MARS function given in \cref{MARS}, letting $n=1000$, $N=2000$ and repeat the process 250 times to get an empirical distribution of the resulting predictions.  The size of each subsample is denoted by $s$ and $k$ denotes the number of nearest neighbors for kNN or the minimum terminal node size for decision trees.  In keeping with our definition of generalized U-statistics, to generate an ensemble, we first simulate $\hat{N}$ and then randomly generate $\hat{N}$ subsamples (without replacement) and build a base learner on each.  Note that we want $\hat{N}\sim \text{Binomial}({n\choose s}, N/{n\choose n})$, but this may be difficult to directly simulate since ${n\choose s}$ could be very large. However, by the Berry-Esseen theorem, we have
	\begin{equation}
		\label{bino-norm}
		\sup_{z\in \R}	\left|\Pr\left\{\frac{\hat{N}- N}{\sqrt{{n\choose s}p(1-p)} }\leq z \right\}- \Phi(z)\right|\leq C\frac{1-2p}{\sqrt{{n\choose s}p(1-p)}},
	\end{equation}
	where $p=N/{n\choose N}$ and $C<0.5$. In our setting $p\ll 10^{-139}$, so the right hand side of \cref{bino-norm} is $\leq {C}/{\sqrt{N}}\approx 0.011$ and $\sqrt{{n\choose s}(p)(1-p)}\approx 1/\sqrt{N}$. Thus, we can simulate $\hat{N}$ by letting $\hat{N} \sim {\tiny }\lfloor{N + \sqrt{N}\cdot \mathcal{N}(0,1)\rfloor}$.
	\begin{figure}[h]
		\centering
		\includegraphics[width=1\textwidth]{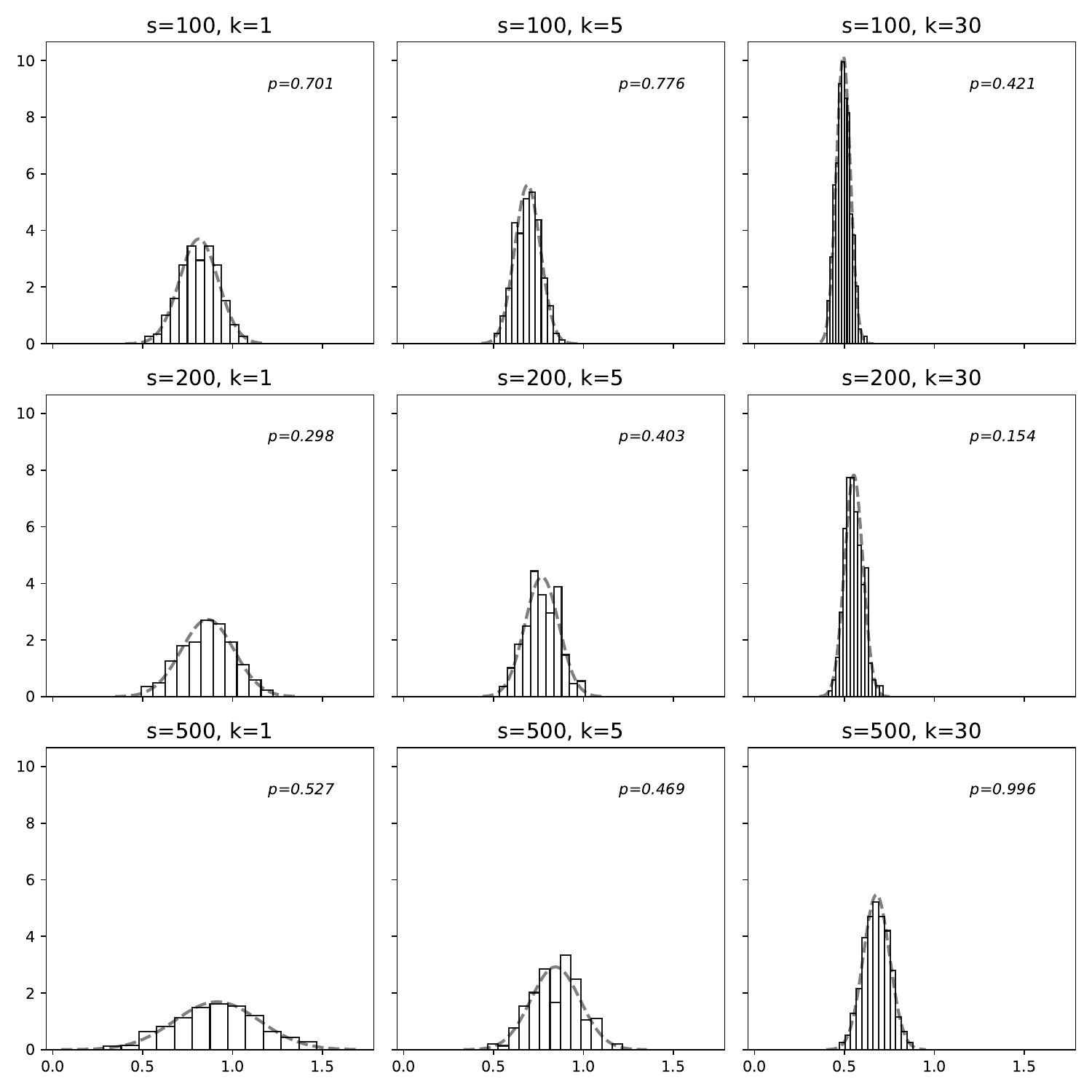}
		\caption{Histograms of predictions from kNN ensembles. Overlaying each
			histogram is the normal density estimate formed by the data.} 
		\label{fig:norm_knn}
	\end{figure} 
	
	Figures \ref{fig:norm_knn} and \ref{fig:norm_tree} show the distribution of predictions for kNN and tree ensembles, respectively, at different subsample sizes $s$ and neighborhood (or terminal node) sizes $k$.  Note that for both ensemble types, the distribution appears normal at every combination of $s$ and $k$.  This was verified with a Shapiro–Wilk test on each; p-values are shown in the upper-right of each plot.  For both ensemble types, we see that the resulting variance of the (approximately normal) distribution seems to increase with $s$ and, at each fixed $s$, decrease as $k$ increases.  Note that for large subsamples, the resulting distributions still remain approximately normal, though they do appear to be getting less normal with larger $s$, which is to be expected.
	\begin{figure}[h]
		\centering
		\includegraphics[width=1\textwidth]{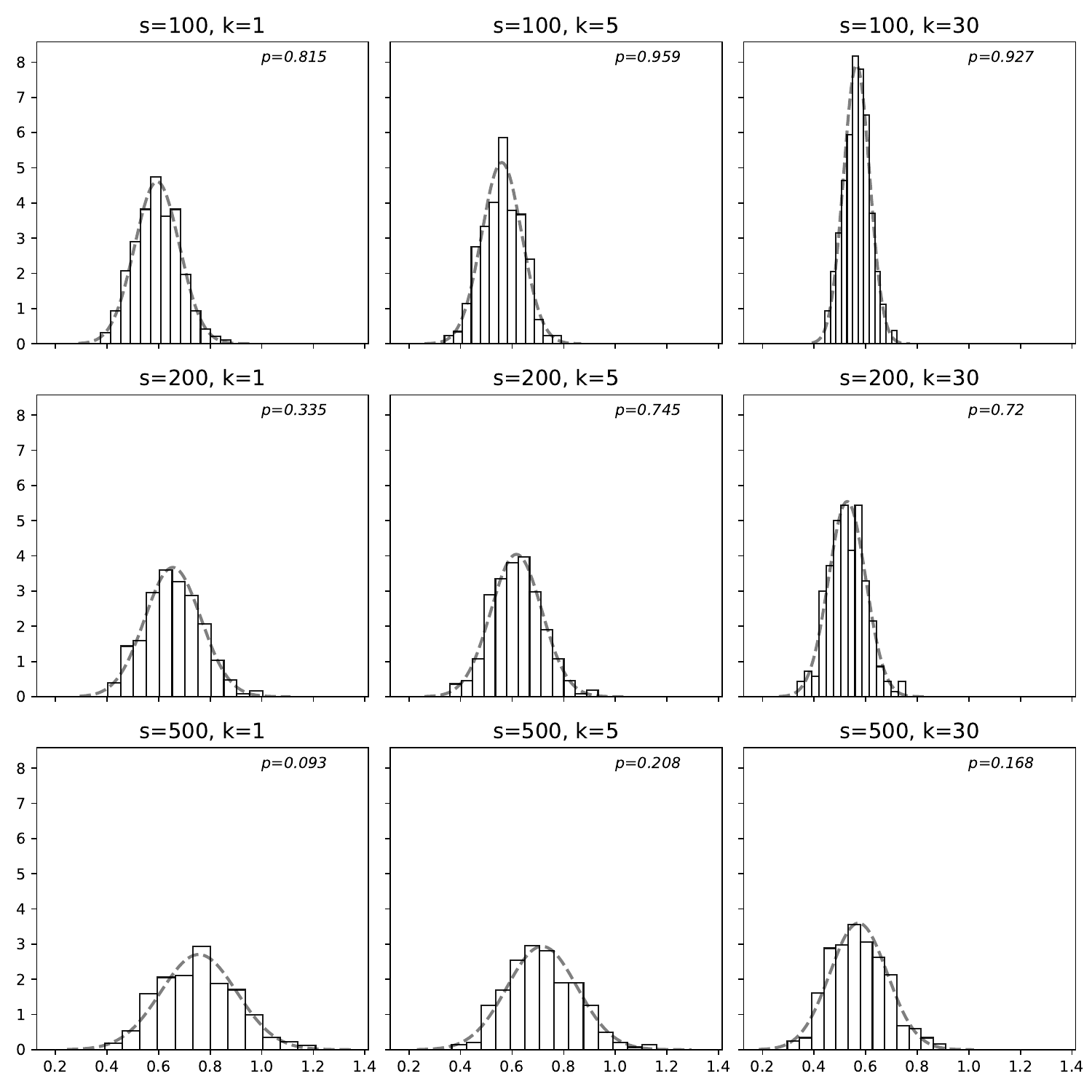}
		\caption{Histograms of predictions from tree-based ensembles. Overlaying each
			histogram is the normal density estimate formed by the data.} 
		\label{fig:norm_tree}
	\end{figure}

	\section{Proofs in \cref{berry-esseen}}
	\label{app:berry-esseen}
	
	\subsection{Introduction to Lemma \ref{chenCol}} 
	\label{app:C-RC}
	
	In many cases of interest, a statistic $T$ can be written as a linear statistic plus a manageable term. \cite{Chen2010} used the K-function approach derived from Stein's method \citep{Stein1972} to build a random concentration inequality for linear statistics. This inequality is an extension of the usual concentration inequalities but the bounds can be random. The authors then apply this randomized concentration inequality to provide a Berry-Esseen bound for $T$ as in Lemma \ref{chenCol}. For completeness, we begin with a brief discussion of this inequality and its derivatives.
	
	Let $Z_1,\ldots, Z_n$ be independent random variables and $T$ be a statistic of the form
	\[
	T=T(Z_1,\ldots,Z_n) = W + \Delta
	\]
	where 
	\[ 
	W = \sum_{i=1}^n g_{n,i}(Z_i), \quad \text{and} \quad \Delta=\Delta(Z_1,\dots, Z_n)
	\]
	for some functions $g_{n,i}$ and $\Delta$. Note here that $W$ is linear and thus $T$ takes the form of a linear statistic plus a remainder. Let $\xi_i = g_{n,i}(Z_i)$ and assume that 
	\begin{equation}
		\label{eq:0}
		\E[\xi_i ] =0 \quad ( i =1,\dots, n) \quad \text{and} \quad  \sum_{i=1}^n \V(\xi_i)=1.
	\end{equation}
	
	The following randomized concentration inequality can be used to establish uniform Berry-Esseen bounds on $T$ with optimal asymptotic rates.
	
	\begin{lemma}[\citep{Chen2010}]
		\label{rce}
		Let $\delta>0$ satisfying
		\begin{equation}
			\label{eq:1}
			\sum_{i=1}^n \E\left[|\xi_i|\min(\delta, |\xi_i|)\right]\geq 1/2.
		\end{equation}
		Then for any real-valued random variables $\Delta_1$ and $\Delta_2$,
		\begin{equation}
			\label{RCE}
			\begin{aligned}
				\P(\Delta_1\leq W\leq \Delta_2)
				& \leq 4\delta + \E|W(\Delta_2-\Delta_1)|\\
				& ~~ + \sum_{i=1}^n\left[ \E|\xi_i(\Delta_1-\Delta_{1,i})| + \E|\xi_i(\Delta_2-\Delta_{2,i})| \right]
			\end{aligned}
		\end{equation}
		whenever $\xi_i$ is independent of $(W-\xi_i,\Delta_{1,i},\Delta_{2,i})$.
	\end{lemma}
	For completeness, we replicate the proof of \cref{RCE} originally given in \cite{Chen2010}. The spirit of the proof is to replace bounding the probability by bounding the expectation of some functions.
	\begin{proof}
		Let \[
		f_{a,b}(w) =\begin{cases}
			-\frac{1}{2}(b-a)-\delta, \quad w<a-\delta \\
			w- \frac{1}{2}(a+b), \quad a-\delta\leq w\leq b+\delta \\
			\frac{1}{2}(b-a) + \delta, \quad w> b+\delta
		\end{cases}
		\]
		and let 
		\[\hat{K}_i(t)= \xi_i\{1_{-\xi_i\leq t\leq 0}-1_{0<t\leq -\xi_i}\}  ,\quad \hat{K}(t)=\sum_{i=1}^n \hat{K}_i(t).\]
		Since $\xi_i$ and $f_{\Delta_{1,i}, \Delta_{2,i}}(W-\xi_i)$ are independent for $1\leq i\leq n$ , we have 
		\[
		\begin{aligned}
			\E \left[Wf_{\Delta_1,\Delta_2}(W) \right]
			& = \sum_{i=1}^n \E\left[\xi_i(f_{\Delta_1,\Delta_2}(W) - f_{\Delta_1,\Delta_2}(W-\xi_i))\right] \\
			& ~~~~ + \sum_{i=1}^n  \E \left[\xi_i(f_{\Delta_1,\Delta_2}(W-\xi_i)-f_{\Delta_{1,i},\Delta_{2,i}}(W-\xi_i))\right]\\
			&  = H_1 + H_2
		\end{aligned}
		\]
		where 
		\[
		\begin{aligned}
			H_1 & = \sum \limits_{i=1}^n \E \left[\xi_i \int_{-\xi_i}^0 f'_{\Delta_1,\Delta_2}(W+t)\,dt\right] \\
			& = \sum \limits_{i=1}^n \E\left[\int_{-\infty}^{\infty} f'_{\Delta_1,\Delta_2}(W+t)\hat{K}_i(t)\,dt\right]\\
			& \geq  \E \left[\int_{|t|\leq \delta} f'_{\Delta_1,\Delta_2}(W+t)\hat{K}(t)\,dt\right]\\
			& \geq  \E\left[1_{\Delta_1\leq W\leq \Delta_2}\int_{|t|\leq \delta}\hat{K}(t)\,dt\right] \\
			& =  \E\left[1_{\Delta_1\leq W\leq \Delta_2} \sum \limits_{i=1}^n |\xi_i|\min(\delta,\xi_i)\right] \\
			&\geq H_{1,1}-H_{1,2}
		\end{aligned}\]
		and where 
		\[
		H_{1,1}= \P(\Delta_1\leq W\leq\Delta_2)\sum  \limits_{i=1}^n \E\left[|\xi_i|\min(\delta,\xi_i)\right]\geq 1/2\P(\Delta_1\leq W\leq \Delta_2)
		\]
		and 
		\[H_{1,2}= \E\left| \sum \limits_{i=1}^n \left[|\xi_i|\min(\delta,\xi_i)- \E\left[|\xi_i|\min(\delta,\xi_i)\right]\right]\right| \leq \V\left(\sum  \limits_{i=1}^n |\xi_i|\min(\delta,\xi_i)\right)^{1/2}\leq \delta.\]
		
		\noindent For $H_2$, we have
		\[
		|f_{\Delta_1,\Delta_2}(w)-f_{\Delta_{1,i},\Delta_{2,i}(w)}|\leq \frac{1}{2}\left|\Delta_1-\Delta_{1,i}\right| + \frac{1}{2}\left|\Delta_2-\Delta_{2,i}\right|
		\]
		which then yields 
		\[
		|H_2|\leq \frac{1}{2}\left( \E|\xi_i(\Delta_1-\Delta_{1,i})| + \E|\xi_i(\Delta_2-\Delta_{2,i})|\right).\]
		It follows from the definition of $f_{a,b}$ that 
		\[ |f_{\Delta_1,\Delta_2}(w)|\leq \frac{1}{2}(\Delta_2-\Delta_1)+\delta.
		\]
		Hence,
		\[
		\begin{aligned}
			\P(\Delta_1\leq W\leq \Delta_2) 
			&  \leq 2 \E\left[ Wf_{\Delta_1,\Delta_2}(W)\right] + 2\delta +  \sum_{i=1}^n
			\left[	\E|\xi_i(\Delta_1-\Delta_{1,i})| + \E|\xi_i(\Delta_2-\Delta_{2,i})|\right] \\
			& \leq \E|W(\Delta_2-\Delta_1)| + 2\delta \E|W| \\
			& ~~~~ + 2\delta +  \sum_{i=1}^n\left[\E|\xi_i(\Delta_1-\Delta_{1,i})| + \E|\xi_i(\Delta_2-\Delta_{2,i})|\right]\\
			&  \leq \E|W(\Delta_2-\Delta_1)| + 4\delta +  \sum_{i=1}^n\left[ \E|\xi_i(\Delta_1-\Delta_{1,i})| + \E|\xi_i(\Delta_2-\Delta_{2,i})|\right].
		\end{aligned}
		\]
		
	\end{proof}
	Now, for any estimator of the form $T = W +\Delta $, we can write
	\[
	-\P(z-|\Delta|\leq W\leq z)\leq \P(T\leq z)-\P(W\leq z)\leq \P(z\leq W\leq z+|\Delta|).
	\]
	Applying \cref{RCE} to these bounds, we arrive at the following lemma. 
	
	\begin{lemma}[\citep{Chen2010}]
		\label{chenCol}
		
		Let $\xi_1,\dots, \xi_n$ be independent random variables satisfying \cref{eq:0}, $W= \sum_{i=1}^n \xi_i$ and  $T = W+\Delta$. Let $\Delta_i$ be a random variable such that $\xi_i$ and $(W-\xi_i, \Delta_i)$ are independent. Then for any $\delta$ satisfying \cref{eq:1}, we have 
		\[
		\sup_{z\in \R}\left|\P(T\leq z)-\P(W\leq z)\right| \leq 4\delta + \E|W\Delta|+ \sum_{i=1}^n
		\E|\xi_i (\Delta-\Delta_i)|.
		\]
		In particular, 
		\begin{equation}
			\label{bound1}
			\sup_{z\in \R}\left|\P(T\leq z)-\P(W\leq z)\right|\leq 2(\beta_2+\beta_3)+ \E|W\Delta|+ \sum_{i=1}^n \E|\xi_i (\Delta-\Delta_i)|
		\end{equation}
		and 
		\begin{equation}
			\label{bound2}
			\sup_{z\in \R}\left|\P(T\leq z)-\Phi(z)\right|\leq 6.1(\beta_2+\beta_3) + \E|W\Delta|+ \sum_{i=1}^n \E|\xi_i (\Delta-\Delta_i)| 
		\end{equation}
		where
		\[
		\beta_2 = \sum \limits_{i=1}^n\E[|\xi_i^2|\one_{|\xi_i|>1}]\quad \quad \text{and} \quad \quad \beta_3 = \sum \limits_{i=1}^n \E[|\xi_i^3|\one_{|\xi_i|\leq 1}] .
		\]
	\end{lemma}
	
	\noindent Note that since $\sum_{i=1}^n \E\left[\xi_i^2\right] =1$, if $\delta>0$ satisfies 
	\[
	\sum \limits_{i=1}^n \E \left[\xi_i^21_{|\xi_i|\geq \delta}\right]<\frac{1}{2}
	\]
	then \cref{eq:1} holds. In particular, when the $\xi_i$ are standardized i.i.d.\ random variables, then $\delta$ must be on the order of $1/\sqrt{n}$.  Furthermore, note that when $\beta_2+\beta_3 \leq 1$ and $4\delta\leq 2(\beta_2+\beta_3)$, then \cref{eq:1} is automatically satisfied and thus \cref{bound1} is immediate.  \cref{bound2} is obtained by combining \cref{bound1} and the sharp Berry-Esseen bound of the sum of independent random variables in \cite{Chen2004}. 
	
	\subsection{Berry-Esseen Bounds for Generalized U-statistics}
	\label{app:C-BR}

	\noindent \textbf{Proof of \cref{be-1}: } We provide the proof for $U_{n,s}$, the extension to $U_{n,s,\omega}$ follows in the same fashion with the only difference being in the H-decomposition. Without loss of generality, let $\theta=0$.
	Observe that 
	\[ 
	U_{n,s} = \sum_{j=1}^s{s \choose j}H_n^{(j)} =  \frac{s}{n}\sum \limits_{i=1}^n g(Z_i)+  \sum_{j=2}^s{s\choose j}H_n^{(j)},
	\]
	where $g(z)=\E[h(z,Z_2,\dots, Z_n)]$ and $H_n^{(j)} = {n \choose j}^{-1}\sum_{(n,j)} h^{(j)}(Z_{i1},\dots, Z_{ij})$. Let 
	\[ \Delta = \sqrt{\frac{{n}}{s^2\zeta_1}} \sum_{j=2}^s{s \choose j}H_n^{(j)} \]
	and for $l\in \{1,\dots, n\}$, let
	\begin{equation}
		\label{eq:delta}
		\Delta_l = \Delta - \sqrt{\frac{{n}}{s^2\zeta_1}} {n \choose j}^{-1}\sum_{S_{j}^{(l)}} h^{(j)}(Z_{i1},\dots, Z_{ij})
	\end{equation}
	where $S_{j}^{(l)}$ denotes the collection of all subsets of variables of size $j$ that include the $l^{th}$ observation.  The choice of $\Delta_l$ plays key role in deciding Berry-Esseen bound. The closer $\Delta_l$ is to $\Delta$, the tighter the bound in \cref{bound2}. We have 
	\begin{equation}
		\label{eq:u=w+delta}
		\sqrt{\frac{n}{s^2\zeta_1}}U_{n,s} = W+\Delta
	\end{equation}
	where $W = \sum_{i=1}^n  \xi_i$ with $\xi_i =g(Z_i)/\sqrt{n\zeta_1}$. For each $i\in \{1,\dots ,n\}$, the random variable $W-\xi_i$ and $\Delta_i$ are functions of $Z_j$, $j\neq i$. Therefore $\xi_i$ is independent of $(W-\xi_i, \Delta_i)$. 
	By Cauchy-Schwarz inequality, we have
	\[ 
	\E[|W\Delta|]\leq \sqrt{\E |W|^2}\cdot \sqrt{\E|\Delta|^2}= \sqrt{\E|\Delta|^2}
	\]
	and
	\[
	\sum_{i=1}^n \E[|\xi_i(\Delta-\Delta_i)|]\leq \sqrt{ \sum_{i=1}^n \E[\xi_i^2]} \cdot \sqrt{ \sum_{i=1}^n \E|\Delta-\Delta_i|^2 }\leq \sqrt{n} \max(\sqrt{ \E|\Delta-\Delta_i|^2}).
	\]
	Observing the terms on the right, we have
	\begin{equation*}
		\begin{aligned}
			\frac{s^2\zeta_1}{n}\E|\Delta|^2 
			&= \V\left\{\sum_{j=2}^s{s\choose j}H_n^{(j)}\right\}\\
			& = \sum_{j=2}^s{s \choose j}^2{n \choose j}^{-1}V_j\\ 
			& \leq \frac{s^2}{n^2}(\zeta_s-s\zeta_1).
		\end{aligned}
	\end{equation*}
	Similarly, we have 
	\begin{equation*}
		\begin{aligned}
			\frac{s^2\zeta_1}{n} \E|\Delta-\Delta_i|^2 
			&  = \V\left\{\sum_{j=2}^s{s \choose j}{n \choose j}^{-1}\sum_{S_j^{(i)}} h^{(j)}(Z_{i1},\dots, Z_{ij})\right\}\\
			& = \sum_{j=2}^s{s \choose j}^2{n \choose j}^{-2}{n-1\choose j-1}V_j \\
			& = \sum_{j=2}^s{s \choose j}^2{n \choose j}^{-1}\frac{j}{n}V_j \\
			& \leq \frac{2s^2}{n^3}\sum_{j=2}^s{s\choose j}V_j\\
			& \leq \frac{2s^2}{n^3}(\zeta_s-s\zeta_1).
		\end{aligned}
	\end{equation*}
	Note that 
	\[
	\begin{aligned}
		\beta_2+\beta_3 
		&= \sum \limits_{i=1}^n \E\left[ \left|\frac{g(Z_i)}{\sqrt{n\zeta_1}}\right|^2 \one_{|g(Z_i)|\geq \sqrt{n\zeta_1}}\right] + \sum \limits_{i=1}^n \E\left[\left|\frac{g(Z_i)}{\sqrt{n\zeta_1}}\right|^3 \one_{|g(Z_i)|\leq \sqrt{n\zeta_1}}\right]\\
		& \leq \frac{1}{n^{1/2}} \frac{\E|g|^3}{\zeta_1^{3/2}}.
	\end{aligned}
	\]
	Finally, by applying Lemma \ref{chenCol}, we obtain 
	\begin{equation*}
		\sup_{z\in \R}\left|\P\left\{\frac{U_{n,s}}{\sqrt{s^2\zeta_1/n}}\leq z\right\}- \Phi(z)\right|\leq \frac{6.1 \E|g|^3}{n^{1/2}\zeta_1^{3/2}} +(1+\sqrt{2})\left\{\frac{s}{n} \left(\frac{\zeta_s}{s\zeta_1}-1\right)\right\}^{1/2}. \quad \blacksquare
	\end{equation*}
	~\\

	\noindent \textbf{Proof of \cref{be-2}: } We provide a bound for incomplete, infinite-order U-statistics. An analogous result for \emph{generalized} incomplete U-statistics $U_{n,s,N,\omega}$ can be established by applying the extended form of the H-decomposition. As eluded to earlier, an incomplete U-statistic can be written as
	\begin{equation}
		\label{eqn:iustat}
		U_{n,s,N} = \frac{1}{\hat{N}} \sum_{(n,s)} \rho h(Z_{i1},\ldots, Z_{is}) 
	\end{equation}
	where  $\rho\sim\text{Bernoulli}(p)$ and $p = N/{n\choose s}$. Note however that \cref{eqn:iustat} can also be written as
	\[
	U_{n,s,N} = \frac{N}{\hat{N}} \left\{ {n\choose s}^{-1}  \sum_{(n,s)} \frac{\rho}{p}h(Z_{i1},\dots, Z_{is})\right\} = \frac{N}{\hat{N}}U^*_{n,s}
	\]
	so that the incomplete U-statistic now takes the form of a scaled, \emph{generalized} complete U-statistic.  We thus now consider $U^*_{n,s}$ and can then extend the results to $U_{n,s,N}$. First, note that the variance terms $\zeta_c^*$ for $c=1,\dots, s$  of $U^*_{n,s}$ are different from those of $U_{n,s}$ in \cref{eq:u}.  For $c=1,\dots, s-1$, we have
	\[
	\zeta^*_c = \Cov(\frac{\rho}{p} h(Z_1,\dots, Z_c,Z_{c+1},\dots, Z_s), \frac{\rho'}{p} h(Z_1,\dots, Z_{c},Z'_{c+1},\dots,Z_n')) = \zeta_c 
	\]
	and 
	\[
	\zeta_s^* = \Cov(\frac{\rho}{p} h(Z_1,\dots, Z_s),\frac{\rho}{p} h(Z_1,\dots, Z_s)) = \frac{1}{p}\zeta_s.
	\]
	The H-decomposition will also be different. Here, we have 
	\begin{align}
		h^{(1)*} &=  \E[\frac{\rho}{p} h \mid Z_1] = h^{(1)} \nonumber \\
		h^{(2)*} &=  \E[\frac{\rho}{p} h \mid Z_1,Z_2] - \E[\frac{\rho}{p} h|Z_1]- \E[\frac{\rho}{p} h \mid Z_2] = h^{(2)} \nonumber \\
		&    \vdots \nonumber \\
		h^{(s)*} &=  \frac{\rho}{p} h - \sum_{j=1}^{s-1}\sum_{(s,j)}h^{(j)}(Z_{i1}, \dots, Z_{ij}) \nonumber
	\end{align}
	\noindent where the $h$ appearing in the earlier form is replaced here by $\frac{\rho}{p} h$.  These kernels still retain the desirable properties laid out in Proposition \ref{prop1}. Furthermore, we have 
	$$V_j^*= V_j \quad \text{for } j = 1,\dots, s-1$$
	and 
	$$V_s^*= \frac{1}{p}\sum_{j=1}^s{s \choose j}V_j - \sum_{j=1}^{s-1}{s \choose j }V_j = V_s+\frac{(1-p)}{p}\zeta_s.$$
	Thus, 
	\[
	\begin{aligned}
		U^*_{n,s} 
		& = \sum_{j=1}^{s-1}{s \choose j}H_n^{(j)*} + H_n^{(s)*} \\
		& = \sum_{j=1}^{s-1}{s \choose j}H_n^{(j)} + {n \choose s}^{-1}\sum_{(n,s)} h^{(s)*}(Z_{i1},\dots, Z_{is}) \\
		& = sH_n^{(1)}+ \Delta.
	\end{aligned}
	\]
	Now, because we have rewritten $U_{n,s}^*$ as a linear term plus a remainder, we can follow the same general strategy as in the complete case above in applying Lemma \ref{chenCol}. In particular,  
	let \[
	\Delta - \Delta_i = \sum \limits_{j=2}^{s-1}{s \choose j}{n \choose j}^{-1}\sum_{S_{j}^{(i)}} h^{(j)}(Z_{i1}, \dots,  Z_{ij})  + {n\choose s}^{-1}\sum_{S_{s}^{(i)}} h^{(s)*}(Z_{i1}, \dots, Z_{is}),
	\]
	where $S_{j}^{(i)}$ denotes the collection of all subsets of size $j$ that include the $i^{th}$ observation.  Then 
	\[
	\begin{aligned}
		\E| \Delta|^2 
		& = \sum_{j=2}^{s-1}{s \choose j}^2{n \choose j}^{-1}V_j +{n \choose s}^{-1}V_s^* \\
		&  = \sum_{j=2}^s{s\choose j}^2{n \choose j}^{-1}V_j + \frac{1}{N}(1-p)\zeta_s,
	\end{aligned}\]
	and 
	\[
	\begin{aligned}
		\E|\Delta-\Delta_i|^2 
		&= \sum_{j=2}^s{s\choose j}^2{n \choose j}^{-1}\frac{j}{n}V_j + {n\choose s}^{-2}{n-1 \choose s-1}V_s^* \\
		& = \sum_{j=2}^s{s \choose j}^2{n \choose j}^{-1}\frac{j}{n}V_j +\frac{s}{n} \frac{1}{N}(1-p)\zeta_s.
	\end{aligned}\]
	Thus
	$$\frac{n}{s^2\zeta_1}\E|\Delta^2|\leq \frac{s}{n}\left[\frac{\zeta_s}{s\zeta_1}-1\right] + \frac{n}{Ns}(1-p)\frac{\zeta_s}{s\zeta_1}$$
	and 
	$$\sum\limits_{i=1}^n  \frac{n}{s^2\zeta_1}\E|\Delta_i^2|\leq \frac{2s}{n}\left[\frac{\zeta_s}{s\zeta_1}-1\right] + \frac{n}{N}(1-p)\frac{\zeta_s}{s\zeta_1}.$$
	By applying \cref{chenCol}, we have
	\[
	\begin{aligned}
		& \sup_{z\in \R}\left|\P\left\{\frac{U_{n,s}^*-\theta}{\sqrt{s^2\zeta_{1}/n}}\leq z\right\}- \Phi(z)\right| \\
		& \leq \frac{6.1\E|g|^3}{n^{1/2}\zeta_{1}^{3/2}} + (1+\sqrt{2})\left\{\frac{s}{n} \left(\frac{\zeta_s}{s\zeta_{1}}-1\right)\right\}^{1/2}  + \left(1+\sqrt{\frac{1}{s}}\right)\left\{ \frac{n}{N}\left(1-p\right) \frac{\zeta_s}{s\zeta_{1}}\right\}^{1/2} \\
		& := \epsilon_0.
	\end{aligned}\]
	
	Next, we give the Berry-Esseen bound for $U_{n,s,N}$. Denote $\frac{U_{n,s,N}-\theta}{\sqrt{s^2\zeta_{1}/n}} = T = W + \Delta$, where $W =\frac{U_{n,s}^*-\theta}{\sqrt{s^2\zeta_{1}/n}}$ and $\Delta = \left(\frac{N}{\hat{N}} - 1\right)W$. Note that we always have 
	\begin{equation*}
		-\P(z-|\Delta|\leq W\leq z)\leq \P(T\leq z)-\P(W\leq z)\leq \P(z\leq W\leq z+|\Delta|).
	\end{equation*}
	Since the Berry-Esseen bound for $W$ is already obtained, it remains only to bound $\P(z\leq W\leq z+|\Delta|)$ and $\P(z-|\Delta|\leq W\leq z)$. Without loss of generality, we consider bounding $\P(z\leq W\leq z+|\Delta|)$.
	By Bernstein's inequality, for any $\epsilon>0$,
	\begin{equation}
		\Pr\left( |\frac{\hat{N}}{N} - 1| \geq \epsilon \right) \leq 2 \exp\left(- \frac{-\epsilon^2 N}{(1-p) + \epsilon/3}\right).
	\end{equation}
	By letting $\epsilon = 2N^{-\beta} < 1/2$ and noting that $\P(|Z|\geq N^\alpha) \leq 2\exp(-N^{2\alpha}/2)$, we have 
	\begin{equation*}
		\begin{aligned}
			\P(|\Delta| \geq N^{-2\beta+\alpha})
			&  \leq \P(|\frac{N}{\hat{N}}-1|\geq N^{-2\beta}) + \P(|W|\geq N^{\alpha}) \\
			& \leq  \P(|\frac{\hat{N}}{N}-1|\geq 2N^{-2\beta}) + \P(|W|\geq N^{\alpha}) \\
			& \leq \P(|\frac{N}{\hat{N}}-1|\geq 2N^{-2\beta}) + \P(|Z|\geq N^\alpha ) + 2\epsilon_0 \\
			& \leq 2\exp(-N^{2\alpha}/2) + 2\exp(-\frac{4N^{1-2\beta}}{(1-p)(1 + 2N^{-\beta}/3)} + 2\epsilon_0 \\
			& := \epsilon_1 + 2\epsilon_0.
		\end{aligned}
	\end{equation*}
	As a result, we have
	\begin{equation*}
		\begin{aligned}
			\P(z\leq W\leq z+|\Delta|) &  \leq  \P(z\leq W\leq z+|\Delta|, \, |\Delta|\leq N^{-2\beta + \alpha}) + \P( |\Delta|\geq N^{-2\beta + \alpha}) \\
			& \leq  \P(z\leq W\leq z+N^{-2\beta+\alpha}) +\epsilon_1 + 2\epsilon_0 \\
			& \leq 2\epsilon_0 + \P(z \leq Z\leq z +N^{-2\beta+\alpha} ) + \epsilon_1 + 2\epsilon_0 \\
			& \leq 4\epsilon_0 + \epsilon_1 + \frac{1}{\sqrt{2\pi}} N^{-2\beta +\alpha} \\
			& := 4\epsilon_0 + \epsilon_1 + \epsilon_2.
		\end{aligned}
	\end{equation*}
	Let $\beta = \frac{1}{2} - \frac{1}{2}\eta_0$ and $\alpha = \frac{1}{2}\eta_0$, where $0 < \eta_0< \frac{1}{2}$. We know that $\epsilon_1 \ll \epsilon_2$ when $N$ is large, and therefore
	\begin{equation*}
		\begin{aligned}
			& \sup_{z\in \R}\left|\P\left\{\frac{U_{n,s,N}-\theta}{\sqrt{s^2\zeta_{1}/n}}\leq z\right\}- \Phi(z)\right| \\
			& =  \sup_{z\in \R}|\Pr(T\leq z) - \Phi(z)|  \\
			& \leq  \sup_{z\in \R}|\P(W\leq z) - \Phi(z)|  +   \sup_{z\in \R}|\Pr(T\leq z) - \P(W\leq z)| \\
			&   \leq 5\epsilon_0 + \epsilon_1 + \epsilon_2 \\
			& = O\left(\frac{\E[g^2]}{n^{1/2}\zeta^{3/2}_{1}} + \left[\frac{s}{n}\left(\frac{\zeta_s}{s\zeta_{1}} - 1\right)\right]^{1/2} +\left\{ \frac{n}{N}(1-p)\frac{\zeta_s}{s\zeta_{1}} \right\}^{1/2}+ N^{-1/2+\eta_0}\right).
		\end{aligned}
	\end{equation*}
	Here, $\eta_0$ can be any constant such that $0 < \eta_0 < 1/2$. 
	Note that the factor of ${N}/{\hat{N}}$ only introduces the extra term - $N^{-1/2+\eta_0}$ in the ultimate Berry-Esseen bound of $U_{n,s,N}$ compared to that of $U_{n,s}^*$.

	\hfill $\blacksquare$. \\

	We now turn to the proof of \cref{be-3}. Before that, we give the statement and proof of  the following simple lemma, which will be used in our proof.
	\begin{lemma}
		\label{lemma-a}
		\begin{equation}
			\lim_{a\to 1^+} \sup_{z\in \R}\left| \frac{\Phi(az)-\Phi(z)}{a-1}  \right| <\infty.
		\end{equation}
	\end{lemma}

	\noindent \textbf{Proof of Lemma \ref{lemma-a}: } Let $f(z) = \Phi(az)-\Phi(z)$, then 
	$$f'(z) = \frac{1}{\sqrt{2\pi}}\left(ae^{-a^2z^2/2} -e^{-z^2/2}\right).$$ Solving $f'(z) = 0$, we get $z^2(a) = \frac{\log(a)}{(a^2-1)/2}$. Then
	\begin{equation}
		\begin{aligned}
			\lim_{a\to 1^+} \sup_{z\in \R}\left| \frac{\Phi(az)-\Phi(z)}{a-1}  \right|
			&  = 	\lim_{a\to 1^+} \left| \frac{\Phi(az(a))-\Phi(z(a))}{a-1}  \right|   \\
			& = \lim_{a\to 1^+} \left|\Phi'(az(a))(az(a))' -\Phi'(z(a))z'(a)\right|.
		\end{aligned}
	\end{equation}
	According to the Taylor expansion of $\log(a)$ at $a=1$, we have 
	$$z^2(a) = \frac{2}{a+1}\left(1-\frac{a-1}{2}+\dots\right) = 1 + o(a-1).$$
	Hence $\lim_{a\to 1^+}z^2(a)=1$ and  $\lim_{a\to 1^+}z'(a) = -\frac{1}{2}$,  and therefore we have 
	\begin{equation}
		\lim_{a\to 1^+} \sup_{z\in \R}\left| \frac{\Phi(az)-\Phi(z)}{a-1}  \right| = \frac{e^{-1/2}}{\sqrt{2\pi}} <\infty
	\end{equation}
	as desired. \hfill $\blacksquare$ \\

	\noindent \textbf{Proof of \cref{be-3}: } We begin with a bound for incomplete, infinite-order U-statistics. The extension of this result to the generalized setting and can be derived in the same fashion. First rewrite \cref{eqn:iustat} as 
	\begin{equation}
		\label{eq:IU}
		U_{n,s,N}= \frac{1}{\hat{N}} \sum_{i } \rho_i h(\bz_i) 
	\end{equation}
	where $\rho_i  \sim \text{Bernoulli}(N/{n \choose s})$ are  i.i.d.\ and $\mathbf{Z}_i= (Z_{i1},\dots, Z_{is})$ denotes a subsample with index $i$ and the sum is taken over all subsamples. 
	We can rewrite $U_{n,s,N}$ in \cref{eq:IU} as a complete U-statistic $U_{n,s}$ plus some manageable term up to the scalar $N/\hat{N}$. We have
	\begin{equation}
		\label{eq:decIU}
		\begin{aligned}
			U_{n,s,N} 
			& = \frac{\hat{N}}{N}\left[{n \choose s}^{-1}\sum_{i}h(\bz_i) + \frac{1}{N}\sum_{i}(\rho_i-p)h(\bz_i)\right] \\
			& = \frac{\hat{N}}{N} \left[U_{n,s} + \left(\sqrt{1-p}\right) \frac{1}{N}\sum_{i}\frac{\rho_i-p}{\sqrt{1-p}}h(\bz_i)\right]  \\
			& = \frac{\hat{N}}{N}\left[A_n + \left(\sqrt{1-p}\right) B_n\right] \\
			& = \frac{\hat{N}}{N}W_n.
		\end{aligned}
	\end{equation}
	We first consider the Berry-Esseen bound of $W_n$. Since we already know the limiting behavior of $A_n$, it remains only to control $B_n$. Note that 
	\[
	\P\left(\sqrt{n}W_n \leq z\right)= \P\left\{\sqrt{N}B_n \leq \frac{z}{\sqrt{\alpha_n(1-p)}}-\sqrt{\frac{N}{1-p}}A_n\right\}
	\]
	where $\alpha_n = n/N$. Conditioning on $Z_1,\dots, Z_n$, $A_n$ can be treated as a constant and we have
	\begin{equation}
		\label{eq:Bn}
		\begin{aligned}
			\sqrt{N}B_n \mid Z_1,\dots, Z_n
			& = {n \choose s}^{-\frac{1}{2}}\sum_{i}\left[\frac{(\rho_i-p)}{\sqrt{p(1-p)}}\right]h(\bz_i) \mid Z_1,\dots, Z_n.
		\end{aligned}
	\end{equation}
	Now, $\sqrt{N}B_n\mid Z_1,\dots, Z_n$ is a sum of independent random variables with variance $U_2$, where 
	\begin{equation}
		\label{eq:gamma}
		U_2=  {n \choose s}^{-1}\sum_{i} h^2(\bz_i).
	\end{equation}
	Let 
	\[
	\xi_i =\frac{(p(1-p))^{-1/2}(\rho_i-p)h(\bz_i)}{\sqrt{\sum_{i}h^2(\bz_i)}},\quad   a_i = \frac{h(\bz_i)}{\sqrt{\sum_{i}h^2(\bz_i)}}
	\]
	then
	\[\sum_{i} a_i^2 = 1\quad \text{and} \quad \xi_i = a_i\left[\frac{(\rho_i-p)}{\sqrt{p(1-p)}}\right].\]
	Applying the  Berry-Esseen bound in \citep{Chen2001} for independent random variables, we have 
	\begin{equation}
		\label{eq:P1}
		\sup_{z\in \R}\left|\P\left(\sqrt{N}B_n\leq z\mid Z_1,\dots, Z_n\right)-\Phi\left(z/\sqrt{U_2}\right)\right| \leq 4.1(\beta_2+\beta_3),
	\end{equation}
	where   $\beta_2 = \sum_{i} \E\left[|\xi_i|^21_{|\xi_i|\geq 1} \right]$ and $\beta_3 = \sum_{i}  \E\left[|\xi_i|^31_{|\xi_i|\leq 1}\right]$.  
	Next,  we show that $(\beta_2+\beta_3)$ can be uniformly bounded by a small number with  high probability and in the rare case when $(\beta_2+\beta_3)$ is large, trivially, we have $$\left|\P\left(\sqrt{N}B_n\leq z\mid Z_1,\dots, Z_n \right)-\Phi\left(z/\sqrt{U_2}\right)\right|\leq 2.$$
	Indeed, 
	\begin{equation}
		\label{beta}
		\begin{aligned}
			\beta_2+\beta_3 
			& \leq \sum_{i} \E|\xi_i|^3 \\
			& = \left\{\frac{ { n \choose s}^{-1}\sum_{i} |h(\bz_i)|^3}{\left({n \choose s}^{-1}\sum_{i}|h(\bz_i)|^2\right)^{3/2}}\right\}\cdot {n \choose s}^{-\frac{1}{2}}\left[ \frac{2p^2-2p+1}{\big( p(1-p)\big)^{1/2}}\right] \\
			& = \frac{U_3}{U_2^{3/2}} {n \choose s}^{-1} \left[\frac{1-2p}{(p(1-p))^{1/2}}\right],
		\end{aligned}
	\end{equation}
	where  $ U_3 = {n \choose s}^{-1}\sum_{i}|h(\bz_i)|^3$. The terms of $U_2$ and $U_3$ are both complete U-statistics and as such, should be  concentrated around their expectations. Let
	\[
	\kappa_1 = \frac{\E|h|^4}{(\E|h|^2)^2}, \quad \kappa_2 = \frac{\E|h|^6}{(\E|h|^3)^2}
	\]
	and recall that $\kappa_1 ,\kappa_2$ are uniformly bounded  by our assumption. Let $\nu_2=\E|h|^2(=\zeta_s)$, $\delta_2 = (\frac{s}{n})^{\eta}\nu_2$, where $\eta>0$. Then  by Chebyshev's inequality, we have 
	\[ 
	\P \left( \left| U_2 - \nu_2\right|\geq \delta_2 \right) \leq \frac{\frac{s}{n}\V \left(|h|^2\right)}{\delta_2^2} =\left(\frac{s}{n}\right)^{1-2\eta} (\kappa_1-1) 
	.\]
	A similar inequality holds for $|U_3-\nu_3|$ and therefore with probability of at least $1 - \epsilon_0$, where $\epsilon_0 = c_0(\frac{s}{n})^{1-2\eta}$ for some constant $c_0>0$, we have
	\[
	\begin{aligned}
		\left|\frac{U_3}{U_2}\right|  = \left|\frac{\frac{U_3}{\nu_3}}{\frac{U_2^{3/2}}{\nu_1^{3/2}}}\cdot \frac{\nu_3}{\nu_2^{3/2}}\right| 
		\leq \left\{\frac{\frac{\nu_3 +\delta_3}{\nu_3}}{\frac{(\nu_2-\delta_2)^{3/2}}{\nu_2^{3/2}}}\right\} \frac{\nu_3}{\nu_2^{3/2}} \leq c_1 \frac{\nu_3}{\nu_2^{3/2}},
	\end{aligned}
	\]
	where $c_1= \left\{\frac{1+(\frac{s}{n})^\eta}{(1-(\frac{s}{n})^\eta)^{3/2  }} \right\} $. Hence, combining this with \cref{beta}, with probability of at least $1-\epsilon_0$, we have 
	\[
	\begin{aligned}
		\beta_2+\beta_3 
		&\leq c_1  \frac{\nu_3}{\nu_2^{3/2}}{n \choose s}^{-\frac{1}{2}}\left\{\frac{1-2p+2p^2}{(p(1-p))^{1/2}}\right\}  \\
		& \leq c_1 \frac{\nu_3}{\nu_2^{3/2}} N^{-\frac{1}{2}}\left\{ \frac{1-2p+2p^2}{(1-p)^{1/2}}\right\}  \\
		& \leq c_1c_2\frac{\nu_3}{\nu_2^{3/2}} N^{-\frac{1}{2}},
	\end{aligned}
	\]
	where $c_2 =  \frac{1-2p+2p^2}{(1-p)^{1/2}} $. The next step is to substitute $U_2$ by $\zeta_s$ by applying  \cref{lemma-a} stated above.
	We obtain 
	\[
	\begin{aligned}
		\sup_{z\in \R}\left|\Phi\left( z/\sqrt{U_2}\right)-\Phi\left(z/\sqrt{\zeta_s} \right)\right| 
		& \leq c_3 |\sqrt{\zeta_s}\wedge \sqrt{U_2}|^{-1}|\sqrt{U_2}-\sqrt{\zeta_s}| \\
		& \leq c_3|\zeta_s \wedge U_2|^{-1}|U_2-\zeta_s| .
	\end{aligned}\]
	Under the event that $|U_2-\zeta_s|\leq \delta_2$, we have
	\begin{equation*}
		\label{eq:P2} 
		\sup_{z\in \R}\left|\Phi\left( z/\sqrt{U_2}\right)-\Phi\left(z/\sqrt{\zeta_s} \right)\right|\leq c_3 \frac{\delta_2}{\zeta_s - \delta_2} 
		=  c_3\frac{\left(\frac{s}{n}\right)^{\eta}}{1- \left(\frac{s}{n}\right)^{\eta}}.
	\end{equation*}
	Next, since $A_n$ is a complete U-statistic, by \cref{be-1}, we have
	\begin{equation}
		\label{eq:P3}
		\sup_{z\in \R}\left|\P\left(\sqrt{n}A_n\leq z\right)-\P\left(Y_A\leq z\right)\right|\leq \epsilon_2
	\end{equation}
	where $\epsilon_2 = \frac{6.1 \E|g|^3}{n^{1/2}\zeta_1^{3/2}} +(1+\sqrt{2})\left\{\frac{s}{n}\left(\frac{\zeta_s}{s \zeta_1}-1\right)\right\}^{1/2}$ and  $Y_A \sim N(0, s^2\zeta_1)$. Lastly,
	\begin{equation*}
		\begin{aligned}
			\P\left(\sqrt{n}W_n\leq z\right)
			& = \E\left[\P\left\{\sqrt{N}B_n\leq \frac{z}{\sqrt{\alpha_n(1-p)}}- \sqrt{\frac{N}{1-p}}A_n \mid Z_1,\dots ,Z_n\right\}\right] \\
			& \leq \P\left\{Y_B \leq \frac{z}{\sqrt{\alpha_n(1-p)}}- \sqrt{\frac{N}{1-p}}A_n\right\} + \epsilon_1 \\
			&  = \P\left\{\sqrt{n}A_n \leq z -\sqrt{\alpha_n(1-p)}Y_B\right\} + \epsilon_1
		\end{aligned}
	\end{equation*}
	where $Y_B \sim N(0, \zeta_s)$ is independent of $Z_1, \dots, Z_n$ and $\epsilon_1 =  4.1\left\{ c_1c_2 \frac{\nu_3}{\nu_2^{3/2}} N^{-1/2}\right\} + c_3\frac{\left(\frac{s}{n}\right)^{\eta}}{1- \left(\frac{s}{n}\right)^{\eta}} + 2\epsilon_0 $. Now, conditioning on $Y_B$, we have 
	\[
	\begin{aligned}
		\P\left\{ \sqrt{n}A_n \leq z -\sqrt{\alpha_n(1-p)}Y_B \mid Y_B\right\}\leq \P\left\{Y_A\leq z -\sqrt{\alpha_n(1-p)}Y_B \mid Y_B\right\} + \epsilon_2.
	\end{aligned}
	\]
	Combining \cref{eq:P1} and \cref{eq:P3}, we conclude that 
	\[
	\begin{aligned}
		\P\left\{\sqrt{n}W_n\leq z\right\}
		&\leq \P\left\{Y_A \leq z- \sqrt{\alpha_n(1-p)}Y_B \right\}+ \epsilon_1+\epsilon_2 \\
		& = \P\left\{Y_A +\sqrt{\alpha_n(1-p)}Y_B\leq z\right\} + \epsilon_1+\epsilon_2 \\
		& \leq \P\left\{Y_A + \alpha_n^{1/2}Y_B\leq z\right\}+  \epsilon_1+\epsilon_2+\epsilon_3.
	\end{aligned}
	\]
	By Lemma \ref{lemma-a}, we have
	\[
	\begin{aligned}\epsilon_3 
		& \leq c_3 \left(s^2\zeta_1 + \alpha_n(1-p)\zeta_s\right)^{-1}\alpha_np\zeta_s\\
		& = c_3 \left(s^2\zeta_1 + \alpha_n(1-p)\zeta_s\right)^{-1}{n\choose s}^{-1}n\zeta_s\\ 
		&\leq c_3 \min\left\{ p(1-p)^{-1}, \frac{n/s}{{n \choose s}}\frac{\zeta_s}{s\zeta_1}\right\}.
	\end{aligned}
	\]
	Thus, in summary,
	\[
	\sup_{z\in \R}\left|\P\left\{\sqrt{N}W_n\leq z\right\}-\P\left\{Y_W\leq z\right\}\right|  \leq \epsilon_1+\epsilon_2+\epsilon_3
	\]
	where 
	\begin{align*}
		\epsilon_1 &= 2c_0\left( \frac{s}{n}\right)^{1-2\eta}   +  4.1\left\{ c_1c_2 \frac{\nu_3}{\nu_2^{3/2}} N^{-1/2} \right\}+ c_3\frac{\left(\frac{s}{n}\right)^{\eta}}{1- \left(\frac{s}{n}\right)^{\eta}} 
		\\
		\epsilon_2 &=  \frac{6.1 \E|g|^3}{n^{1/2}\zeta_1^{3/2}} +(1+\sqrt{2})\left\{\frac{s}{n} \left(\frac{\zeta_s}{s\zeta_1}-1\right)\right\}^{1/2}\\
		\epsilon_3 &= c_3 \cdot \min\left\{ p(1-p)^{-1}, \frac{n/s}{{n \choose s}}\frac{\zeta_s}{s\zeta_1}\right\}
	\end{align*}
	
	\noindent and $Y_W\sim N(0, {s^2}\zeta_1/n + \zeta_s/N)$.   Note that $\epsilon_1$ and $\epsilon_2$ dominate because of the $\binom{n}{s}$ in the denominator of $\epsilon_3$. Therefore, by choosing $\eta=1/3$, the above bound can be simplified as 
	\[
	\begin{aligned}
		\epsilon_1+\epsilon_2+\epsilon_3  
		& \leq C\left\{ \frac{\E|g|^3}{n^{1/2}(\E|g|^2)^{3/2}} + \frac{\E|h|^3}{N^{1/2}(\E|h|^2)^{3/2}} \right. \\
		&~~~~~~~~~+ \left.\left\{\frac{s}{n} \left(\frac{\zeta_s}{s\zeta_1}-1\right)\right\}^{1/2}+\left(\frac{s}{n}\right)^{1/3} \right\}. 
	\end{aligned} 
	\]
	As in the Proof of \cref{be-2}, the scalar $\hat{N}/ N$ introduces the extra term $N^{-1/2+\eta_0}$ in the ultimate Berry-Essen bound of $U_{n,s,N}$, where $0 < \eta_0<1/2$.
	\hfill $\blacksquare$ 
	~\\

	\subsection{Discussion on A Tighter Bound}
	\label{app:C-TB}
	
	Here we provide a sketch of the proof of \cref{be-4}. Let $m = \lfloor n/s\rfloor$ and define
	\begin{equation*}
		V(Z_1,Z_2,\dots, Z_n) = \frac{1}{m}\sum_{j=0}^{m-1} h(Z_{j\cdot s+1},Z_{j\cdot s+2}\dots, Z_{j\cdot s+s}).
	\end{equation*}
	The general form of a complete U-statistic in \cref{eq:u} can be rewritten as
	\begin{equation*}
		U_{n,s} = \frac{1}{n!} \sum_{\beta \in S_n}V(Z_{\beta_1, \beta_2, \dots , \beta_n})
	\end{equation*}
	where $S_n$ consists of all permutations of $(1,2,\dots,n)$. Now, suppose that $(h-\theta)/\sigma$ is sub-Gaussian with variance proxy $v^2$, where $\sigma^2= \V(h)$, then by definition, we have
	\begin{equation}
		\label{subGaussian}
		\E\left[\exp(\lambda (h-\theta))\right] \leq \exp\left\{ \frac{\lambda^2 \sigma^2 v^2}{2}\right\},\quad \lambda\in \R
	\end{equation}
	and hence we have 
	\begin{align*}
		\P(U_{n,s}-\theta>t) 
		& \leq \exp(-\lambda t)\E[\exp\left(\lambda(U_{n,s}-\theta)\right)]\\
		&  \leq \exp(-\lambda t)\sum_{\beta \in S_n} \frac{1}{n!}\E[\exp\left(\lambda(V(Z_{\beta_1},Z_{\beta_2},\dots,Z_{\beta_n})-\theta)\right)]  \\
		& = \exp(-\lambda t) \E[\exp\left(\lambda(V-\theta)\right)] \label{u-v}    \stepcounter{equation}\tag{\theequation} \\
		& \leq \exp\left\{-\frac{mt^2}{2\sigma^2 v^2}\right\}, \quad t>0.
	\end{align*}
	The second inequality in \cref{u-v} is due to Jensen's inequality and the last inequality is due to Hoeffding inequality.  Observe that $\P\left(U_{n,s}-\theta< t \right)$ follows in the same manner (recall that \cref{subGaussian} holds for all $\lambda \in \R$), and we get
	\begin{equation}
		\label{eq5.3}
		\P\left(\left|U_{n,s}-\theta\right|\geq t\right)\leq 2\exp\left\{-\frac{mt^2}{2\sigma^2 v^2}\right\}.
	\end{equation}
	Now let $t = m^{-1/2+\eta}\sigma$ where $0<\eta<1/2$.  Then with probability at least $1-2\exp\left(-\frac{1}{2v^2}(\lfloor n/s\rfloor)^{2\eta}\right)$, $|U_{n,s}-\theta|\leq (\lfloor n/s\rfloor )^{1/2-\eta}\sigma$.  Therefore if $|h-\theta|^2$ and $|h-\theta|^3$ are sub-Gaussian after being standardized, we can then apply \cref{eq5.3} in the proof of \cref{be-3} to obtain the improved result.

\end{document}